\documentclass{article}
\usepackage{titletoc}

\usepackage[ruled,vlined,linesnumbered]{algorithm2e} %
\makeatletter
\newcommand{\nosemic}{\renewcommand{\@endalgocfline}{\relax}}%
\newcommand{\dosemic}{\renewcommand{\@endalgocfline}{\algocf@endline}}%
\let\oldnl\nl%
\newcommand{\nonl}{\renewcommand{\nl}{\let\nl\oldnl}}%
\makeatother

\PassOptionsToPackage{sort,numbers, compress}{natbib}

\usepackage[final]{neurips_2022}

\usepackage[utf8]{inputenc} %
\usepackage[T1]{fontenc}    %
\usepackage{hyperref}       %
\usepackage{url}            %
\usepackage{booktabs}       %
\usepackage{amsfonts}       %
\usepackage{nicefrac}       %
\usepackage{microtype}      %
\usepackage{xcolor}         %

\usepackage{caption}
\usepackage{subcaption}
\usepackage{float}

\usepackage{amsmath,amssymb}
\allowdisplaybreaks

\usepackage{amsthm}

\newtheorem{assumption}{Assumption}
\newtheorem{theorem}{Theorem}
\numberwithin{theorem}{section}
\newtheorem{corollary}[theorem]{Corollary}
\newtheorem{proposition}[theorem]{Proposition}
\newtheorem{lemma}[theorem]{Lemma}

\newtheorem{innercustomgeneric}{\customgenericname}
\providecommand{\customgenericname}{}
\newcommand{\newcustomtheorem}[2]{%
  \newenvironment{#1}[1]
  {%
   \renewcommand\customgenericname{#2}%
   \renewcommand\theinnercustomgeneric{##1}%
   \innercustomgeneric
  }
  {\endinnercustomgeneric}
}
\newcustomtheorem{customassumption}{Assumption}

\usepackage{xfrac}

\newcommand{\Identity}{{\rm I\kern-.2em l}}
\newcommand{\Expect}{\mathbb{E}}
\newcommand{\Expectbracket}[1]{\mathbb{E}\left[ #1 \right]}
\newcommand{\Expectsubbracket}[2]{\mathbb{E}_{#1}\left[ #2 \right]}
\newcommand{\Expectcond}[2]{\mathbb{E}\left[\left. #1 \right| #2 \right]}

\newcommand{\x}{\mathbf{x}}
\newcommand{\y}{\mathbf{y}}
\newcommand{\z}{\mathbf{z}}
\newcommand{\g}{\mathbf{g}}

\newcommand{\bu}{\mathbf{u}}
\newcommand{\period}{P}
\newcommand{\divergence}{d}
\newcommand{\probconst}{c}

\newcommand{\norm}[1]{\left\Vert #1 \right\Vert}
\newcommand{\normsq}[1]{\left\Vert #1 \right\Vert^2}
\newcommand{\innerprod}[1]{\left\langle #1 \right\rangle}

\usepackage{tabularx}
\usepackage{multirow}
\usepackage{makecell}
\usepackage{multicol}

\usepackage{enumitem}

\makeatletter
\newcommand{\thickhline}{%
    \noalign {\ifnum 0=`}\fi \hrule height 1pt
    \futurelet \reserved@a \@xhline
}
\newcolumntype{"}{@{\hskip\tabcolsep\vrule width 1pt\hskip\tabcolsep}}
\makeatother

\title{A Unified Analysis of Federated Learning with Arbitrary Client Participation
}

\author{%
  Shiqiang Wang \\
  IBM T. J. Watson Research Center\\
  Yorktown Heights, NY 10598 \\
  \texttt{wangshiq@us.ibm.com} \\
  \And
  Mingyue Ji \\
  Department of ECE,
  University of Utah\\
  Salt Lake City, UT 84112 \\
  \texttt{mingyue.ji@utah.edu} \\
}

\begin{document}

\maketitle

\begin{abstract}
Federated learning (FL) faces challenges of intermittent client availability and computation/communication efficiency. As a result, only a small subset of clients can participate in FL at a given time. It is important to understand how partial client participation affects convergence, but most existing works have either considered idealized participation patterns or obtained results with non-zero optimality error for generic patterns. In this paper, we provide a unified convergence analysis for FL with arbitrary client participation. We first introduce a generalized version of federated averaging (FedAvg) that amplifies parameter updates at an interval of multiple FL rounds. Then, we present a novel analysis that captures the effect of client participation in a single term. By analyzing this term, we obtain convergence upper bounds for a wide range of participation patterns, including both non-stochastic and stochastic cases, which match either the lower bound of stochastic gradient descent (SGD) or the state-of-the-art results in specific settings. We also discuss various insights, recommendations, and experimental results. 
\end{abstract}

\section{Introduction}
\label{sec:introduction}

We consider a federated learning (FL) problem with $N$ clients \cite{kairouz2019advances,li2020federated,yang2019federated}. Our goal is to minimize:
\begin{align}
\textstyle
    f(\x) := \frac{1}{N}\sum_{n=1}^N  F_n(\x),
    \label{eq:FLObj}
\end{align}
where $\x\in\mathbb{R}^m$ is an $m$-dimensional model parameter. The local objective $F_n(\x)$ of each client~$n$ is usually defined as $F_n(\x) := \Expect_{\xi_n\sim\mathcal{D}_n}\left[\ell_n(\x, \xi_n)\right]$, where $\mathcal{D}_n$ is the distribution of client $n$'s local data, which cannot be observed globally because the data remains private, and $\ell_n(\x, \xi)$ is the per-sample loss function for parameter $\x$ and input data $\xi$. The objective \eqref{eq:FLObj} can be extended to a weighted average, but we do not write out the weights and consider them as part of $\ell_n(\x, \xi)$ and $F_n(\x)$.
A canonical way of solving \eqref{eq:FLObj} is to use federated averaging (FedAvg) \cite{mcmahan2017communication}, which is a form of stochastic gradient descent (SGD) that operates in multiple rounds, where each \textit{round} includes multiple local update steps followed by communication between clients and a server to synchronize the updates. 

\textbf{Partial Participation.} A major challenge in FL is that clients only intermittently participate in the collaborative training process \cite{Bonawitz2019Towards}. The reason is twofold. First, FL applications usually have a large pool of clients. It is impractical to have all clients participating in all FL rounds, because it would require computation and communication that can consume a large amount of energy and also lead to network congestion. Second, clients may become unavailable from time to time. For example, if the client devices are smartphones, they may be willing to participate in FL only when they are charging (usually at night), to avoid draining the battery in outdoor environments where no charger is available.

\textbf{Open Problem.} 
The convergence of FL algorithms with realistic client participation patterns has not been fully understood. 
Most existing works on partial participation have only considered idealized scenarios, such as those where clients are randomly selected according to a given probability distribution that is independent across rounds~\cite{chen2020optimal,fraboni2021clustered,fraboni2021impact,karimireddy2020scaffold,li2020federatedOptimization,Li2020On,yang2021achieving}. They cannot incorporate cases where  clients are unavailable for a period of time. %
Some recent results considering general scenarios either include a constant error term that does not vanish at convergence \cite{yang2021anarchic,cho2022towards} or
are otherwise less competitive than state-of-the-art results obtained in idealized settings
\cite{gu2021fast,yan2020distributed}. 
Therefore, we ask:\\ 
\textit{1. What classes of participation patterns have guaranteed convergence to zero error?}\\ 
\textit{2. For these classes, can the convergence rate match with that obtained in idealized scenarios?}

To answer these questions, we need to overcome two challenges. First, we need to devise a \textit{unified analytical methodology} for obtaining the convergence rates under various participation patterns. Second, we need a \textit{general and realistic algorithm} for cross-device FL, with appropriate control options that can be configured to make the convergence rate competitive.

\begin{table*}[t]
\caption{Summary and comparison of results (only showing baselines without constant error term)}
\label{table:mainResults}
\small
{
\centering
\renewcommand{\arraystretch}{2}
\begin{tabular}{>{\centering\arraybackslash}p{0.1\linewidth} | >{\centering\arraybackslash}p{0.12\linewidth} | >{\centering\arraybackslash}p{0.4\linewidth} | >{\centering\arraybackslash}p{0.25\linewidth} } 
\thickhline
\textbf{Method} & \textbf{Participation} & \textbf{Convergence error upper bound} & \textbf{Remark}\\
\thickhline
\!\!\!Yan~et~al.~\cite{yan2020distributed}\!\!\!
&\makecell{Unavailable \\ $\leq E$ rounds} & $\mathcal{O}\left(\frac{(1+\sigma^2)N^{\sfrac{1}{4}}\sqrt{E}}{\sqrt{ST}}\right)$ & \makecell{\textit{Add'l assumption:} \\ bounded gradient norm}\\
\hline
\!\!\!\!Gu et al. \cite{gu2021fast}\!\!\!\!
&\makecell{Unavailable  \\ $\leq E$ rounds} & $\mathcal{O}\left(\frac{(1+\sigma^2)\sqrt{E}}{\sqrt{NIT}}\right)$ & 
\!\!\!\!\makecell{\textit{Add'l assumptions:} Hessian \\ Lipschitz, a.s. bounded noise}\!\!\!\! \\
\hline
\multirow{4}{*}{\makecell{\textbf{Ours} \\ (this paper)}} & Regularized & $\mathcal{O}\left(\frac{\sigma}{\sqrt{SIT}}\right)$ & \makecell{Cor.~\ref{corollary:regularizedGiven}, matching \\ \!\!\!centralized SGD lower bound\!\!\!}\\
\cline{2-4}
& Ergodic & Approaches zero as $T\rightarrow\infty$ & Prop.~\ref{prop:ergodic}\\
\cline{2-4}
& %
Mixing & $\mathcal{O}\!\left(\!\frac{1+\sigma^2}{\sqrt{SIT}}\!\right)$ i.E.;\, $\mathcal{O}\!\left(\!\frac{1+\sigma^2}{c\sqrt{SIT}}\!\right)$ w.p. $1\!-\!c$ & \makecell{\!\!Prop.~\ref{prop:mixing} \& Cor.~\ref{corollary:nonRegularizedExpectation},\!\! \\ matching idealized bound}\\
\cline{2-4}
& Independent & $\mathcal{O}\!\left(\!\frac{1+\sigma^2}{\sqrt{SIT}}\!\right)$ i.E.;\, $\mathcal{O}\!\left(\!\frac{(1+\sigma^2)\ln\left(\sfrac{2}{c}\right)}{\sqrt{SIT}}\!\right)$ w.p. $1\!-\!c$ & \makecell{\!\!Prop.~\ref{prop:independent} \& Cor.~\ref{corollary:nonRegularizedExpectation},\!\! \\ matching idealized bound}\\
\thickhline
\end{tabular}
}
Note: \textit{We let $\rho=\frac{1}{\sqrt{S}}$ in our results for a direct comparison with other works}, see Section~\ref{sec:discussions} for more details. We consider sufficiently large $T$ in all cases and only show the dominant term in this table, with other variables ignored in $\mathcal{O}(\cdot)$. Definitions: i.E. = in expectation, w.p. = with probability, a.s. = almost surely, $S$: number of participating clients in each round ($S\leq N$), other definitions are given throughout the paper.
\end{table*}

\textbf{Our Contributions.} In this paper, we overcome these challenges and give answers to the two questions above. 
We first introduce a generalized version of FedAvg, which \textit{amplifies} the parameter updates after every $\period$ rounds, for some $\period\geq 1$. When $\period=1$, our setting is the same as FedAvg with two-sided learning rates \cite{karimireddy2020scaffold,yang2021achieving}, but we focus on an extended setup where $\period$ can be greater than one.
For this generalized FedAvg algorithm, we perform a new convergence analysis that \textit{unifies} the effect of partial participation into a single term ($\tilde{\delta}^2(\period)$ in Section~\ref{sec:convergenceAnalysis}). This allows us to \textit{decouple} our main convergence analysis from the analysis on partial participation, which largely simplifies the analytical procedure. Our novel methodology gives the results shown in Table~\ref{table:mainResults}, for a broad range of participation patterns. These patterns include \textit{non-stochastic and regularized} participation, where all clients participate equally but not simultaneously within $\period$ rounds, and \textit{stochastic} participation settings including \textit{ergodic}, \textit{stationary and strongly mixing} (e.g., Markov process), and \textit{independent}. For the stochastic participation, we provide both expected and high probability convergence bounds, where the expected rates match the state-of-the-art FedAvg convergence rates that were derived for the idealized setting of random participation following an independent distribution.
Furthermore, we provide new insights related to non-IID data, amplification of updates, ``linear speedup'' \cite{yu2019linear}, etc.

To summarize, our main contributions are as follows:
\begin{itemize}[itemsep=-0.2em,topsep=-0.2em]
    \item We introduce a generalized FedAvg algorithm which amplifies the updates aggregated over multiple rounds, where the amplification interval $P$ can be tuned for the best convergence.
    \item We present novel analysis and unified methodology for obtaining competitive convergence upper bounds with arbitrary client participation patterns. 
    \item We discuss important insights from both the theoretical and experimental results.
    \vspace{0.3em}
\end{itemize}

\textbf{Related Work.}
FedAvg \cite{mcmahan2017communication} is characterized by partial client participation and multiple local updates  in each round. In the case of full participation, FedAvg is also known as local SGD  \cite{Gorbunov2021Local,Haddadpour2019Local,Lin2020Dont,stich2018local,MLSYS2019Jianyu,CooperativeSGD} and parallel restarted SGD \cite{yu2019parallel}. Partial participation was considered in \cite{chen2020optimal,fraboni2021clustered,fraboni2021impact,karimireddy2020scaffold,li2020federatedOptimization,Li2020On,yang2021achieving}, with the assumption that clients are selected to participate probabilistically according to a given distribution that is independent across rounds. 
Some recent works have started to incorporate unavailable clients. For example, the work in \cite{ruan2021towards} allows inactive clients, but the theoretical result does not guarantee convergence if a client participates with zero probability in a certain round. Similarly, the result in \cite{yang2021anarchic,cho2022towards} includes a constant error term in the case of general participation in cross-device FL.
Other works have obtained convergence rates related to the maximum number of inactive rounds, based on additional assumptions such as Hessian Lipschitz \cite{gu2021fast} and bounded gradient norm \cite{yan2020distributed}. Their theory also requires all clients to participate once at the beginning (see \cite[Remark 5.2]{gu2021fast}).
In contrast, our result in this paper is based on a minimal set of assumptions. We consider a simple algorithm without requiring initial full participation, and we provide a better rate of convergence to zero error (see Table~\ref{table:mainResults}). 
Moreover, existing methods usually have separate analyses and results for each specific class of participation pattern, whereas we present a unified framework that largely simplifies the analytical procedure. 
Some more discussions on related works are in Appendix~\ref{sec:appendixAdditionalRelatedWork}.

\section{Generalized FedAvg with Amplified Updates}

We consider a generalized version of FedAvg as shown in Algorithm~\ref{alg:main-alg}. 
In this algorithm, we define $q_t^n$ to be the \textit{participation weight} of client $n$ in round $t$. If $q_t^n = 0$, the client does not participate in this round. Each client $n$ participates whenever $q_t^n > 0$, where the weight $q_t^n$ is applied in the global update step in Line~\ref{alg-line:globalUpdate}. For mathematical convenience in our analysis, we assume in Algorithm~\ref{alg:main-alg} that all clients compute their local updates in Lines~\ref{alg-line:loopWorkers}--\ref{alg-line:compressedUpdateVector}, regardless of whether they participate in the current round or not. We emphasize that this is \textit{logically equivalent} to the practical setting where non-participating clients do not compute, because their computations have no effect in subsequent rounds when $q_t^n = 0$ (see Line~\ref{alg-line:globalUpdate}). In every round $t$, each client $n$ computes $I$ steps of local updates according to Line~\ref{alg-line:localUpdate}, where $\gamma > 0$ is the local learning rate and $\g_n(\y_{t,i}^n)$ is the stochastic gradient of $F_n(\y_{t,i}^n)$ such that $\Expectcond{\g_n(\y_{t,i}^n)}{\y_{t,i}^n} = \nabla F_n(\y_{t,i}^n)$.

Starting at round $t_0$, the global updates are accumulated in $\bu$ for $\period$ rounds (Line~\ref{alg-line:accumulate}). 
At the end of every interval of $\period$ rounds, an amplification is applied to the accumulated updates (Line~\ref{alg-line:amplify}). Since $\x_t$ is first updated without amplification in rounds $t_0,\ldots,t_0 + \period - 1$, Line~\ref{alg-line:amplify} adds $\bu$ after multiplying $\eta-1$, which is equivalent to setting $\x_{t+1}$ to $\x_{t_0} + \eta\bu$. When $\eta=1$, there is no amplification and Line~\ref{alg-line:amplify} has no effect in this case. In general, we allow any $\eta>0$ including values of $\eta$ that are less than one. When $\eta<1$, we reduce (instead of amplify) the updates. Intuitively, amplification with $\eta>1$ is usually preferred over reduction, because we can compute the gradients on similar model parameters in each round where only a small number of clients participates. After $\period$ rounds, the accumulated updates provide a better estimation for the overall client population, and amplifying the updates allows the model parameter to progress towards the descent direction for the majority of clients (see Appendix~\ref{sec:appendixMotivatingExample} for an example). Also note that although $P$ is a parameter in our algorithm, the dominant terms of our convergence upper bounds shown in Table~\ref{table:mainResults} \textit{do not} depend on~$P$.

\begin{algorithm}[h]
 \caption{Generalized FedAvg with amplified updates and arbitrary participation}
 \label{alg:main-alg}
 
 \SetKwProg{Indent}{}{}{}
 \SetKwIF{CustomFor}{CustomForUnused}{CustomForUnused2}{for}{do}{}{}{}

\begin{multicols}{2}
\textbf{Input:} $\gamma$, $\eta$, $\x_0$, $I$, $\period $, $T$;\,\,
\textbf{Output:} $\{\x_t : \forall t\}$

Initialize
$t_0 \leftarrow 0$,
$\bu \leftarrow \mathbf{0}$;

\uCustomFor{$t = 0, \ldots, T-1$}{

    \For{$n = 1,\ldots,N$ in parallel \label{alg-line:loopWorkers}}{
    
            $\y^n_{t,0} \leftarrow \x_t$;  \label{alg-line:assignGlobalParameter}
            
            \For{$i = 0, \ldots, I-1$ \label{alg-line:loopLocal}}{
                $\y^n_{t,i+1} \leftarrow \y^n_{t,i} - \gamma \g_n(\y^n_{t,i})$ \label{alg-line:localUpdate}
            }
            
            $\Delta_t^n \leftarrow \y^n_{t,I} - \x_t$;\, \label{alg-line:compressedUpdateVector} 

    }
    
}

\nonl ~

\nonl ~

\nonl 
\Indent{}{

    $\x_{t+1} \leftarrow \x_t + \sum_{n=1}^N q_t^n \Delta_t^n$; \label{alg-line:globalUpdate}\, //update

    $\bu \leftarrow  \bu + \sum_{n=1}^N q_t^n \Delta_t^n$; \label{alg-line:accumulate}\, //accumulate

    \If{$t+1-t_0 = \period $}{
        $\x_{t+1} \leftarrow \x_{t+1} + (\eta - 1)\bu$; \label{alg-line:amplify}\,  //amplify
    
        $t_0 \leftarrow t+1$;

        $\bu \leftarrow \mathbf{0}$;

    }
}

\end{multicols}
\end{algorithm}

\section{Convergence Analysis and Main Result}
\label{sec:convergenceAnalysis}

We analyze Algorithm~\ref{alg:main-alg} as follows.
For ease of analysis, we assume that $\sum_{n=1}^N q_t^n = 1$ for $q_t^n \geq 0$ for all $t$.\footnote{When $\sum_{n=1}^N q_t^n \neq 1$, the algorithm is equivalent to the case of $\{q_t^n\}$ normalized to one and a global learning rate applied to the updates in each round. Since we already amplify the updates every $\period $ rounds, we do not apply a separate global learning rate to each round.}
We also define $\rho$ such that $\sum_{n=1}^N (q_t^n)^2 \leq \rho^2$ for all $t$.
As $\sum_{n=1}^N q_t^n = 1$, there always exists $\rho$ such that $\rho^2 \leq 1$.
We define $\mathcal{Q} := \{q_t^n : \forall n, t\}$. 
Mathematically, we use the conditional expectation $\Expectcond{\cdot}{\mathcal{Q}}$ to denote the case where $\mathcal{Q}$ is given and the expectation is only over the stochastic gradient noise. In the case where $\mathcal{Q}$ is stochastic, the full expectation $\Expectbracket{\cdot}$ is taken over both the noise and $\mathcal{Q}$.
We make the following assumptions that are commonly used in the literature.

\begin{assumption}[Lipschitz gradient]
    \label{assumption:Lipschitz}
    \begin{align}
        \norm{\nabla F_n(\x) - \nabla F_n(\y)} \leq L \norm{\x - \y}, \forall \x, \y, n.
    \end{align}
\end{assumption}

\begin{assumption}[Unbiased stochastic gradient with bounded variance]
    \label{assumption:gradientNoise}
    \begin{align}
        \Expectcond{\g_n(\x)}{\x} = \nabla F_n(\x)  
        \textnormal{ and }
        &\Expectcond{\normsq{\g_n(\x) - \nabla F_n(\x)}}{\x} \leq \sigma^2, \forall \x, n.
    \end{align}
\end{assumption}

\begin{assumption}[Bounded gradient divergence]
    \label{assumption:gradientDivergence}
    \begin{align}
        \normsq{\nabla F_n(\x) - \nabla f(\x)} \leq \divergence^2, \forall \x, n.
    \end{align}
\end{assumption}

% It is also assumed that the stochastic gradient noise is independent across time and across clients, as commonly done in existing works.
The gradient divergence in Assumption~\ref{assumption:gradientDivergence} is related to the degree of non-IID data distribution across clients.
To better interpret how different divergence components affect the convergence, we introduce an alternative set of divergence bounds in Assumption~\ref{assumption:gradientDivergenceAlternative} as follows. We will show later (in Section~\ref{sec:interpret}) that if Assumption~\ref{assumption:gradientDivergence} holds, then Assumption~\ref{assumption:gradientDivergenceAlternative} also holds with properly chosen $\tilde{\beta}^2$, $\tilde{\nu}^2$, and $\tilde{\delta}^2(\period)$.

\begin{customassumption}{\protect\NoHyper\ref{assumption:gradientDivergence}'\protect\endNoHyper}[Alternative gradient divergence bound]
    \label{assumption:gradientDivergenceAlternative}
    \begin{align}
        &\textstyle \normsq{\sum_{n=1}^N q_t^n\left[\nabla F_n(\x) - \nabla f(\x)\right]} \leq \tilde{\beta}^2, \forall \x, t, \label{eq:divergenceBeta}\\
        &\textstyle \sum_{n=1}^N q_t^{n}\normsq{\nabla F_n(\x) - \sum_{n'=1}^N q_t^{n'} \nabla F_{n'}(\x)} \leq \tilde{\nu}^2, \forall \x, t, \label{eq:divergenceNu}\\
        &\textstyle \normsq{\frac{1}{\period } \sum_{t={t_0}}^{{t_0}+\period -1} \sum_{n=1}^N  q_t^n \left(\nabla F_n(\x) - \nabla f(\x)\right)}  \leq  \tilde{\delta}^2(\period), \forall \x, t_0. \label{eq:divergenceDelta}
    \end{align}
\end{customassumption}

In the above, $\tilde{\beta}^2$ and $\tilde{\nu}^2$ capture the gradient divergence in an arbitrary round $t$, and $\tilde{\delta}^2(\period)$ captures the divergence of time-averaged gradients over $\period $ rounds, which is a function of $\period$. In particular, $\tilde{\beta}^2$ is an upper bound of the divergence between the original objective's gradient $\nabla f(\x)$ and the averaged gradient among participating clients weighted by $\{q_n^t\}$. When all the clients participate, we have $q_t^n = \frac{1}{N}$ and \eqref{eq:divergenceBeta} holds with $\tilde{\beta}^2 = 0$. The quantity $\tilde{\nu}^2$ is an upper bound of the divergence between each client's gradient and the averaged gradient of participating clients, where the average is also weighted by $\{q_n^t\}$. We will show in Section~\ref{subsec:betaNuDiscussion} that the overall divergence $\divergence^2$ in Assumption~\ref{assumption:gradientDivergence} can be expressed as a sum of $\tilde{\beta}^2$ and $\tilde{\nu}^2$, as intuition suggests. The quantity $\tilde{\delta}^2(\period)$ in \eqref{eq:divergenceDelta} extends $\tilde{\beta}^2$ in \eqref{eq:divergenceBeta} by computing the average over $\period $ rounds (inside the norm) instead of a single round. When the weights $\{q_n^t\}$ are properly chosen, as $\period $ gets large, the average $q_n^t$ over $\period $ rounds gets close to each other for different $n$, and $\tilde{\delta}^2(\period)$  becomes small. We will formally analyze this behavior in Section~\ref{subsection:deltaDiscussion}.

In the following, let $\mathcal{F} := f(\x_0) - f^*$, where $f^*$ is the true minimum of \eqref{eq:FLObj}, i.e., $f^* := \min_\x f(\x)$. 

\begin{theorem}
\label{theorem:mainConvergenceResult}
When Assumptions~\ref{assumption:Lipschitz}, \ref{assumption:gradientNoise}, and \ref{assumption:gradientDivergenceAlternative} hold, $\gamma \leq \frac{1}{12 LI\period }$, $\gamma \eta \leq \frac{1}{2LI\period }$, and $\period \leq \frac{T}{2}$, we have
\begin{align}
    \!\!\!\!&\!\!\!\!\min_{t} \Expectcond{\normsq{\nabla f(\x_{{t}})}}{\mathcal{Q}} \nonumber \\
    &\!\!\!\!\leq  \mathcal{O}\Bigg( \frac{\mathcal{F}}{\gamma \eta IT} +   \gamma^2 L^2 I^2 \tilde{\nu}^2 +   \gamma^2 L^2 I^2 \period ^2  \tilde{\beta}^2 +   \tilde{\delta}^2(\period)  +\left( \gamma^2 L^2 I  +   \gamma^2 L^2 I\period   \rho^2  +  \gamma\eta L \rho^2\right) \sigma^2  \Bigg).
\end{align}
\end{theorem}
\begin{proof}[Proof Sketch]
    We first use smoothness to relate $f(\x_{t_0+P})$ and $f(\x_{t_0})$, which includes an inner-product term that can be expanded into several terms. One of these terms includes $\nabla F_n(\y^n_{t,i}) - \nabla f(\x_{{t_0}})$. We upper bound this term with three other terms including $\normsq{ \y^n_{t,i} - \sum_{n'=1}^N q_t^{n'} \y^{n'}_{t,i}}$, $\normsq{  \sum_{n'=1}^N q_t^{n'} \y^{n'}_{t,i} - \x_{{t_0}}}$, and $\tilde{\delta}^2(\period)$. The upper bounds of the first two terms are found by solving recurrence relations, where the recurrence for the second term includes the first term. This nested recurrence is a uniqueness in this proof.
    The full proof is given in Appendix~\ref{sec:appendixProof}. 
\end{proof}

For two specific (but different) learning rate configurations, we can obtain the following corollaries.

\begin{corollary}
\label{corollary:mainConvergenceResult1}
Choosing $\gamma = \frac{1}{12LI\period \sqrt{T}}$ and $\eta = \min\left\{\frac{12P\sqrt{LI\mathcal{F}}}{\sigma\rho} ; 6\sqrt{T}\right\}$, for $\period \leq\frac{T}{2}$, we have
\begin{align}
    \quad&\min_{t}\, \Expectcond{\normsq{\nabla f(\x_{{t}})}}{\mathcal{Q}} \leq  \mathcal{O}\Bigg( \frac{\sigma\rho\sqrt{L\mathcal{F}}}{\sqrt{IT}} \!+\! \frac{L\period \mathcal{F}}{T}  \!+\!  \frac{\tilde{\nu}^2}{\period ^2 T} \!+\! \frac{\tilde{\beta}^2}{T}\! +  \! \tilde{\delta}^2(\period) \! + \!\frac{\sigma^2 }{I\period T}\Bigg).
\end{align}
\end{corollary}

\begin{corollary}
\label{corollary:mainConvergenceResult2}
If $\frac{\sqrt{\mathcal{F}}}{\rho\sqrt{LIT}} \leq \frac{1}{2LI\period }$, choosing $\gamma = \frac{1}{12LI\period \sqrt{T}}$ and $\eta = \frac{12P\sqrt{LI\mathcal{F}}}{\rho}$, for $\period \leq\frac{T}{2}$, we have
\begin{align}
    &\min_{t}\, \Expectcond{\normsq{\nabla f(\x_{{t}})}}{\mathcal{Q}}  \leq  \mathcal{O}\Bigg(\! \frac{(1\!+\!\sigma^2)\rho\sqrt{L\mathcal{F}}}{\sqrt{IT}} \!+\!  \frac{\tilde{\nu}^2}{\period ^2 T} \!+\! \frac{\tilde{\beta}^2}{T}\! +  \! \tilde{\delta}^2(\period) \! + \!\frac{\sigma^2 }{I\period T}\!\Bigg).\!
\end{align}
\end{corollary}

These two corollaries are key building blocks of our \textit{unified convergence analysis framework}. The choice between them depends on whether and how fast $P$ grows with $T$. We will also see in the next section that $\tilde{\delta}^2(\period)$ is key in capturing the effect of different participation patterns. After deriving $\tilde{\delta}^2(\period)$ such that (\ref{eq:divergenceDelta}) holds, plugging it back to one of the above corollaries gives the results in Table~\ref{table:mainResults}.

\section{Interpreting and Applying the Unified Framework}
\label{sec:interpret}

\subsection{\texorpdfstring{Discussion on $\tilde{\beta}^2$ and $\tilde{\nu}^2$: }{}Decomposition of Divergence}
\label{subsec:betaNuDiscussion}

In Corollaries~\ref{corollary:mainConvergenceResult1} and \ref{corollary:mainConvergenceResult2}, we have an interesting observation that the term involving $\tilde{\nu}$ decreases with $\period$, whereas the term involving $\tilde{\beta}$ does not decrease with $\period$. This suggests that when $\period$ is large, the divergence term $\tilde{\beta}$ has the main effect on convergence. 
We recall that $\tilde{\nu}$ captures the divergence between the gradient of each client and the average gradient of clients that participate in a round $t$. Since the learning rate $\gamma$ is inversely proportional to both $I$ and $\period$, a large $\period$ gives a smaller learning rate which makes the divergence (related to $\tilde{\nu}$) after $I$ iterations in each round less significant. For the term with $\tilde{\beta}$, increasing $\period$ does not have the same effect, because parameter averaging is only conducted across the participating clients instead of all the clients, and $\tilde{\beta}$ captures the divergence with respect to all the clients. The effect of all clients' participation pattern is captured by the term $\tilde{\delta}^2(\period)$.

From Assumption~\ref{assumption:gradientDivergence}, by adding and subtracting $\sum_{n'=1}^N q_t^{n'} \nabla F_{n'}(\x)$ inside the norm and expanding the square, we can obtain
\begin{align}
    &\textstyle\sum_{n=1}^N q_t^{n}\normsq{\nabla F_n(\x) - \nabla f(\x)} \nonumber \\ %
    &\textstyle = \sum_{n=1}^N q_t^{n}\normsq{\nabla F_n(\x) - \sum_{n'=1}^N q_t^{n'} \nabla F_{n'}(\x)} + \normsq{\sum_{n'=1}^N q_t^{n'} \nabla F_{n'}(\x) - \nabla f(\x)} 
    \leq \divergence^2,
    \label{eq:betaNuExpression}
\end{align}
because $\sum_{n=1}^N q_t^{n} = 1$ and the inner-product term is zero. The two terms in the right-hand side (RHS) of \eqref{eq:betaNuExpression} are equal to the left-hand side (LHS) of \eqref{eq:divergenceBeta} and \eqref{eq:divergenceNu}, respectively. We thus have the following result showing that $\tilde{\beta}^2$ and $\tilde{\nu}^2$ can be seen as a decomposition of the original divergence $\divergence^2$.
\begin{proposition}
    \label{prop:divergenceDecomposition}
    There exist $\tilde{\beta}^2$ and $\tilde{\nu}^2$, such that $\tilde{\beta}^2 + \tilde{\nu}^2 \leq \divergence^2$ while satisfying \eqref{eq:divergenceBeta} and \eqref{eq:divergenceNu}.
\end{proposition}

\subsection{\texorpdfstring{Discussion on $\tilde{\delta}^2(\period)$: }{}Effect of Partial Participation}
\label{subsection:deltaDiscussion}

The results in Theorem~\ref{theorem:mainConvergenceResult} and Corollaries~\ref{corollary:mainConvergenceResult1} and \ref{corollary:mainConvergenceResult2} include a term of $\tilde{\delta}^2(\period)$. To guarantee convergence to zero error, $\tilde{\delta}^2(\period)$ needs to be either equal to zero or decrease in $T$.

\textbf{\textit{Condition for $\tilde{\delta}^2(\period) = 0$.}}
From \eqref{eq:FLObj}, we know that $0 = \frac{1}{N}\sum_{n=1}^N  F_n(\x) - f(\x) = \frac{1}{N}\sum_{n=1}^N  \left(F_n(\x) - f(\x)\right)$. Hence,
\begin{align}
    \textstyle \mu_{t_0}\sum_{n=1}^N \left( F_n(\x) - f(\x)\right)=0
    \label{eq:muEqualsZero}
\end{align}
for an arbitrary constant $\mu_{t_0}$, which can possibly depend on $t_0$. Then, we obtain
\begin{align}
    & \textstyle \normsq{\frac{1}{\period } \sum_{t={t_0}}^{{t_0}+\period -1} \sum_{n=1}^N q_t^n \left(\nabla F_n(\x) - \nabla f(\x)\right)} 
    = \normsq{\sum_{n=1}^N \frac{1}{\period } \sum_{t={t_0}}^{{t_0}+\period -1} q_t^n \left(\nabla F_n(\x) - \nabla f(\x)\right)} \nonumber\\
    &\textstyle\quad\quad\quad\quad\quad\quad\quad\quad\quad = \normsq{\sum_{n=1}^N\!\! \left[\!\frac{1}{\period }\! \sum_{t={t_0}}^{{t_0}+\period -1}\! q_t^n \!-\! \mu_{t_0}\!\right] \!\left(\nabla\! F_n(\x)\! -\! \nabla\! f(\x)\right)} \!,\!
    \label{eq:deltaExpression}
\end{align}
where the last equality is due to \eqref{eq:muEqualsZero}. Since \eqref{eq:deltaExpression} is equal to the LHS of \eqref{eq:divergenceDelta}, we have the following.
\begin{proposition}[Regularized participation]
    \label{prop:regularized}
    If $\frac{1}{\period } \sum_{t={t_0}}^{{t_0}+\period -1} q_t^n - \mu_{t_0} = 0$ for some $\mu_{t_0}$ and all $n\in\{1,\ldots,N\}$, $t_0\in\{0, P, 2P,\ldots\}$, then~\eqref{eq:divergenceDelta} holds with $\tilde{\delta}^2(\period) = 0$.
\end{proposition}
This implies that if $\period$ is chosen 
so that the averaged participation weights over every interval of $\period$ rounds ($\frac{1}{\period } \sum_{t={t_0}}^{{t_0}+\period -1} q_t^n$) are equal to each other among all the clients $n\in\{1,\ldots,N\}$, i.e., regularized, then $\tilde{\delta}^2(\period) = 0$. 
Together with Corollary~\ref{corollary:mainConvergenceResult1} and Proposition~\ref{prop:divergenceDecomposition}, we have the following.

\begin{corollary}[Regularized participation]
\label{corollary:regularizedGiven}
If the condition of Proposition~\ref{prop:regularized} holds for $\period\leq\frac{T}{2}$, and we choose the learning rates according to Corollary~\ref{corollary:mainConvergenceResult1}, then
\begin{align}
    \min_{t} \Expectcond{\normsq{\nabla f(\x_{{t}})}}{\mathcal{Q}} 
    & \leq  \mathcal{O}\Bigg( \frac{\sigma\rho\sqrt{L\mathcal{F}}}{\sqrt{IT}} + \frac{L\period \mathcal{F} + d^2 + \sigma^2}{T}\Bigg).
\end{align}
\end{corollary}

If Proposition~\ref{prop:regularized} holds for a finite $\period$ that does not depend on $T$, we obtain a convergence rate of $\mathcal{O}\left(\frac{\sigma\rho}{\sqrt{IT}} + \frac{1+\sigma^2}{T}\right)$ for $T\geq 2P$, where all the other constants are ignored in the $\mathcal{O}\left(\cdot\right)$ notation here. For full gradient descent with $\sigma^2=0$, this convergence rate is improved to $\mathcal{O}\left(\frac{1}{T}\right)$.

\textbf{\textit{Interpreting $\tilde{\delta}^2(\period)$ Using Variance.}}
The condition for $\tilde{\delta}^2(\period) = 0$ may be too stringent for some practical scenarios of FL. Next, we show that $\tilde{\delta}^2(\period)$ can be expressed as the variance of client participation weights averaged over $\period$ rounds.
We first express $\tilde{\delta}^2(\period)$ with $\divergence^2$ by further bounding:
\begin{align}
    &\textstyle \eqref{eq:deltaExpression} \leq N \sum_{n=1}^N \left[\overline{q_{t_0}^n} - \mu_{t_0}\right]^2 \divergence^2 ,
    \label{eq:deltaExpressionVariance}
\end{align}
where we first use Jensen's inequality on the sum over $n$ in \eqref{eq:deltaExpression}, then move the term $\overline{q_{t_0}^n} - \mu_{t_0}$ to the outside of the norm, where we define $\overline{q_{t_0}^n}  := \frac{1}{\period } \sum_{t={t_0}}^{{t_0}+\period -1} q_t^n$ for any $n\in\{1,\ldots,N\}$, and afterwards bound $\normsq{\nabla F_n(\x) - \nabla f(\x)}$ by $\divergence^2$ according to Assumption~\ref{assumption:gradientDivergence}. We note that \eqref{eq:divergenceDelta} holds if the RHS of \eqref{eq:deltaExpressionVariance} is upper bounded by $\tilde{\delta}^2(\period)$. Therefore, we can obtain %
$\tilde{\delta}^2(\period)$ by analyzing the statistical properties of $\overline{q_{t_0}^n} - \mu_{t_0}$ and then choosing $\tilde{\delta}^2(\period)$ to be equal to the RHS of \eqref{eq:deltaExpressionVariance}.

In the following, we assume that $\mu_{t_0}$ is chosen such that $\mu_{t_0} = \Expectbracket{q_t^n}$ for all $n\in\{1,\ldots,N\}$ and $t\in\{t_0, t_0+1,\ldots, t_0 +P-1\}$. This implies that the mean participation weights of all the clients within the same cycle of $\period$ rounds are equal. Note that this condition is the same as unbiased client sampling in existing works \cite{fraboni2021clustered,fraboni2021impact,karimireddy2020scaffold,li2020federatedOptimization,Li2020On,yang2021achieving}, because we consider an unweighted average in \eqref{eq:FLObj} and any weighting is included in the local objective function $F_n(\x)$. However, differently from these existing works, we do not assume independence here (we will only assume independence in a special case later). Under this condition, it is apparent that $\mathrm{Var}\left(\overline{q_{t_0}^n}\right) = \Expectbracket{\left(\overline{q_{t_0}^n}  - \mu_{t_0}\right)^2}$ is the variance of $q_t^n$ averaged over $\period$ rounds, for any $n\in\{1,\ldots,N\}$.
We immediately have the following result.
\begin{proposition}[Ergodic participation]
    \label{prop:ergodic}
    If $\{q_t^n : \forall t\}$ is a mean-ergodic process for any $n\in\{1,\ldots,N\}$, which means that
    $\textstyle\lim_{\period\rightarrow\infty}\Expectbracket{\left(\overline{q_{t_0}^n}  - \mu_{t_0}\right)^2}= 0$ for any $n$ and $t_0$,
    then there exists $\tilde{\delta}^2(\period)$ such that $\lim_{\period\rightarrow\infty}\tilde{\delta}^2(\period) = 0$ while satisfying \eqref{eq:divergenceDelta} in expectation.
\end{proposition}

Next, we consider the class of stationary and strongly mixing processes. Informally, a random process is said to be \textit{strongly mixing} with coefficient $\alpha(\period)$ if the outcomes that are at least $\period$ steps apart are nearly independent, where the distance between the joint probability distribution of the outcomes and their independent counterpart is at most $\alpha(\period)$. See \cite[Sec. 27]{billingsley1995probability} for a formal definition. A specific example of stationary and strongly mixing processes is finite-state irreducible and aperiodic Markov chains, which have an $\alpha(\period)$ that exponentially decreases in $\period$ \cite[Thm. 8.9]{billingsley1995probability}. 
It is known that the following (generalized) central limit theorem holds for stationary and strongly mixing processes, where we adapt the result to our participation weights.
\begin{lemma}[\mbox{[\citenum{billingsley1995probability}, Thm. 27.4]}]
    \label{lemma:CLT}
    If $\{q_t^n : \forall t\}$ is stationary and strongly mixing with $\alpha(\period)=\mathcal{O}(\period^{-5})$, for any $n\in\{1,\ldots,N\}$, then $\overline{q_{t_0}^n} \sim \mathcal{N}\left(\mu_{t_0}, \frac{\hat{\upsilon}^2}{\period}\right)$ for any $t_0$ when $\period\rightarrow\infty$, where $\mathcal{N}(\cdot,\cdot)$ denotes the normal distribution and $\hat{\upsilon}^2 := \mathrm{Var}(q_{t_0}^n) + 2\sum_{p=1}^\infty\mathrm{Cov}(q_{t_0}^n, q_{t_0+p}^n)$.
\end{lemma}
From the definition of strongly mixing, we know that $\mathrm{Cov}(q_{t_0}^n, q_{t_0+p}^n)$ in the definition of $\hat{\upsilon}^2$ approaches zero as $p$ gets large, hence $\sum_{p=1}^\infty\mathrm{Cov}(q_{t_0}^n, q_{t_0+p}^n)$ converges to a finite value. When $\period\rightarrow\infty$ and $\alpha(\period)=\mathcal{O}(\period^{-5})$, Lemma~\ref{lemma:CLT} shows that $\mathrm{Var}\left(\overline{q_{t_0}^n}\right) = \frac{\hat{\upsilon}^2}{\period}$. Chebyshev's inequality gives
\begin{align}
\textstyle 
    \Pr\left\{\left(\overline{q_{t_0}^n}  - \mu_{t_0}\right)^2 \leq \frac{\hat{\upsilon}^2}{\probconst\period} \right\} \geq 1-\probconst.
\end{align}
Plugging  $\mathrm{Var}\!\left(\overline{q_{t_0}^n}\right) \!= \!\Expectsubbracket{\!\!}{\!\left(\overline{q_{t_0}^n} \! -\! \mu_{t_0}\right)^{\!2}} \!= \!\frac{\hat{\upsilon}^2}{\period}$ and $\left(\overline{q_{t_0}^n}\!  - \!\mu_{t_0}\right)^2 \!\leq \!\frac{\hat{\upsilon}^2}{\probconst\period}$ into the RHS of \eqref{eq:deltaExpressionVariance} gives the following expected and high-probability results, respectively, where ``w.p.'' stands for ``with probability''.
\begin{proposition}[Stationary and strongly mixing participation]
    \label{prop:mixing}
    If $\{q_t^n : \forall t\}$ is stationary and strongly mixing with $\alpha(\period)=\mathcal{O}(\period^{-5})$, for any $n\in\{1,\ldots,N\}$, then as $\period\rightarrow\infty$:\\
    Choosing $\tilde{\delta}^2(\period)=\frac{N^2d^2\hat{\upsilon}^2}{\period}$ satisfies \eqref{eq:divergenceDelta} in expectation;  $\tilde{\delta}^2(\period)=\frac{N^2d^2\hat{\upsilon}^2}{\probconst\period}$ satisfies \eqref{eq:divergenceDelta} w.p. $1-\probconst$.
\end{proposition}

For independent participation, we can obtain a similar result while \textit{not} requiring $\period\rightarrow\infty$.
Assuming that $\mathrm{Var}(q_t^n) \leq \upsilon^2$ for any $t$ and $n$, it is evident that $\mathrm{Var}\left(\overline{q_{t_0}^n}\right) \leq \frac{\upsilon^2}{\period}$, because the variance of the sum of independent random variables is equal to the sum of the variance, and $\mathrm{Var}(aZ) = a^2 \mathrm{Var}(Z)$, for an arbitrary constant $a$ and random variable~$Z$.
We note that Hoeffding's inequality \cite{hoeffding1994probability} gives
\begin{align}
\textstyle 
    \Pr\left\{\left(\overline{q_{t_0}^n}  - \mu_{t_0}\right)^2 \leq \frac{\ln\left(\sfrac{2}{\probconst}\right)}{2\period} \right\} \geq 1-\probconst.
\end{align}
We have the following result by plugging into the RHS of~\eqref{eq:deltaExpressionVariance}.

\begin{proposition}[Independent participation]
    \label{prop:independent}
    If $\{q_t^n : \forall t\}$ is independent across $t$, for any $n$, then:\\
    Choosing $\tilde{\delta}^2(\period)=\frac{N^2d^2\upsilon^2}{\period}$ satisfies \eqref{eq:divergenceDelta} in expectation;  $\tilde{\delta}^2(\period)=\frac{N^2d^2\ln\left(\sfrac{2}{\probconst}\right)}{2\period}$ satisfies \eqref{eq:divergenceDelta} w.p. $1\!-\!\probconst$.
\end{proposition}

Note that the independence here is only assumed across $t$, so $q_t^n$ and $q_t^{n'}$ for the same $t$ but $n\neq n'$ can still be dependent of each other. Compared to Proposition~\ref{prop:mixing}, the high-probability bound in Proposition~\ref{prop:independent} has $c$ in the logarithmic term, which is tighter  since $\frac{1}{2}\ln\left(\frac{2}{c}\right) \leq \frac{1}{c} - \frac{1}{2}$ (ignoring $\hat{\upsilon}^2$).

\textbf{\textit{Choosing $\period$ as a Function of $T$.}}
The choices of $\tilde{\delta}^2(\period)$ in Propositions~\ref{prop:mixing} and \ref{prop:independent} include $\period$ in the denominator. Hence, we can guarantee convergence to zero error if we choose $\period$ to be an increasing function of $T$, which also ensures that $\period\rightarrow\infty$ as $T\rightarrow\infty$. If we choose $\period \propto N^{\frac{5}{2}} \sqrt{IT}$, we can obtain the following convergence rate from Corollary~\ref{corollary:mainConvergenceResult2} together with Proposition~\ref{prop:divergenceDecomposition}.
\begin{corollary}[Stochastic participation]
\label{corollary:nonRegularizedExpectation}
If the conditions of either Proposition~\ref{prop:mixing} or Proposition~\ref{prop:independent} hold, and we choose $\period = \hat{\upsilon}^2 N^{\frac{5}{2}} \sqrt{IT}$ or $\period = \upsilon^2 N^{\frac{5}{2}} \sqrt{IT}$, respectively, and the learning rates according to Corollary~\ref{corollary:mainConvergenceResult2}, then the convergence error in expectation over $\mathcal{Q}$ satisfies
\begin{align}
    \min_{t} \Expectbracket{\normsq{\nabla f(\x_{{t}})}} 
    &\leq  \mathcal{O}\Bigg( \frac{(1+\sigma^2)\rho\sqrt{L\mathcal{F}}}{\sqrt{IT}} + \frac{\divergence^2}{\sqrt{NIT}} +  \frac{\divergence^2 + \sigma^2}{T} \Bigg),
\end{align}
for $I\leq \frac{\rho}{2\hat{\upsilon}^2 N^{\sfrac{5}{2}}\sqrt{L\mathcal{F}}}$ and $T\rightarrow\infty$ when using conditions of Proposition~\ref{prop:mixing}, or $I\leq \frac{\rho}{2\upsilon^2 N^{\sfrac{5}{2}}\sqrt{L\mathcal{F}}}$ and $T\geq 4\upsilon^4IN^5$ when using condition of Proposition~\ref{prop:independent}.
\end{corollary}

\textbf{Remark.}
Although the upper bound of $I$ in Corollary~\ref{corollary:nonRegularizedExpectation} may appear restrictive, note that the optimal solution $\x^*$ to \eqref{eq:FLObj} remains the same when the objective $f(\x)$ is scaled by a multiplicative positive constant, which in turn scales $\sqrt{L\mathcal{F}}$. 
More specifically, consider an arbitrary $a>0$,
we define $F'(\x) := a F(\x)$ and $f'(\x) := \frac{1}{N}\sum_{n=1}^N F'_n(\x) = a f(\x)$. Then, we have $f'(\x_0)-f'^* = af(\x_0)-af^* = a\mathcal{F}$. We also have $\norm{\nabla F'_n(\x) - \nabla F'_n(\y)} = a\norm{\nabla F_n(\x) - \nabla F_n(\y)} \leq aL \norm{\x - \y}$ from Assumption~\ref{assumption:Lipschitz}, due to the linearity of gradients. Hence, by choosing $a\in(0,1)$ and replacing $F(\x)$ with $F'(\x)$, we can make $\sqrt{L\mathcal{F}}$ arbitrarily small and $\frac{\rho}{2\hat{\upsilon}^2 N^{\sfrac{5}{2}}\sqrt{L\mathcal{F}}}$ or $\frac{\rho}{2\upsilon^2 N^{\sfrac{5}{2}}\sqrt{L\mathcal{F}}}$ (the upper bound of~$I$ in Corollary~\ref{corollary:nonRegularizedExpectation}) arbitrarily large without affecting the optimal solution $\x^*$. Thus, we can potentially allow arbitrarily large $I$ by scaling $f(\x)$ without changing the optimal solution $\x^*$. 
We leave the in-depth study of this phenomenon for future work, where we recognize that such a scaling will affect the LHS of the convergence bound too, although the optimal solution $\x^*$ does not change.

We also give the following further insight related to Corollary~\ref{corollary:nonRegularizedExpectation}. In Corollary~\ref{corollary:nonRegularizedExpectation}, we choose $\period$ to be proportional to the variance $\hat{\upsilon}^2$ or $\upsilon^2$. This is intuitive because when the clients' participation weights have higher variance, we would like to wait for more rounds before amplifying, so that the contributions of clients are more balanced. In the same way, $\period$ increases with the number of clients $N$. Recall that $\sum_{n=1}^N q_t^n = 1$. Intuitively, when the participation patterns of individual clients remain ``unchanged'', $q_t^n$ scales with $\frac{1}{N}$ and its variance $\hat{\upsilon}^2$ or $\upsilon^2$ scales with $\frac{1}{N^2}$, so the product $\hat{\upsilon}^2 N^{\frac{5}{2}} \sim \sqrt{N}$ or $\upsilon^2 N^{\frac{5}{2}} \sim \sqrt{N}$, implying that $P$ effectively scales with $\sqrt{N}$ when this intuition holds.

While the result in Corollary~\ref{corollary:nonRegularizedExpectation} is in expectation, the high-probability bound is similar (see Table~\ref{table:mainResults}).
In general, since \eqref{eq:divergenceDelta} is used as an upper bound in the proof of Theorem~\ref{theorem:mainConvergenceResult},
if $\tilde{\delta}^2(\period)$ is chosen so that \eqref{eq:divergenceDelta} holds in expectation or with a certain probability, such as in Propositions~\ref{prop:ergodic}, \ref{prop:mixing} and \ref{prop:independent}, then the final convergence result also holds in expectation or with a certain probability, respectively.
% \vspace{-0.5em}

\section{\texorpdfstring{Discussions and Insights}{Discussions and Insights}}
\label{sec:discussions}

\textbf{Interpretation of ``Linear Speedup''.}
Linear speedup is a desirable property seen in existing works that consider idealized client participation \cite{yu2019linear,yu2019parallel,yang2021achieving}. It essence, it means that the same convergence error can be achieved by increasing the number of participating clients $S$ (with $S\leq N$) and reducing the number of rounds $T$, while keeping the product $ST$ unchanged. In our case of arbitrary client participation, the coefficient $\rho$ is a generalization of the $\sfrac{1}{\sqrt{S}}$ term for linear speedup in existing works that select a fixed number of $S$ clients in each round.
When $S$ clients participate with equal weight, because $\sum_{n=1}^N q_t^n = 1$, we have $q_t^n=\sfrac{1}{S}$ for the participating clients and $q_t^n=0$ for the non-participating clients, so $\rho=\big[\sum_{n=1}^N (q_t^n)^2\big]^{\sfrac{1}{2}}=\sfrac{1}{\sqrt{S}}$. Plugging this $\rho$ back to the convergence results, we achieve a convergence rate of $\mathcal{O}\left(\sfrac{\sigma}{\sqrt{SIT}}\right)$ in Corollary~\ref{corollary:regularizedGiven} and $\mathcal{O}\left(\sfrac{(1+\sigma^2)}{\sqrt{SIT}}\right)$ in Corollary~\ref{corollary:nonRegularizedExpectation}, for sufficiently large $T$ while ignoring the other variables in $\mathcal{O}(\cdot)$, as shown in Table~\ref{table:mainResults}. This shows that we can achieve linear speedup in $S$ in this special case. We can also generalize to having at least $S_\mathrm{min}$ clients participating with equal weight in each round. Following the same argument, we have $\rho \leq \sfrac{1}{\sqrt{S_\mathrm{min}}}$ in this case, giving a linear speedup in $S_\mathrm{min}$. In general, we achieve a linear speedup factor of $\sfrac{1}{\rho^2}$ in the case of arbitrary participation.

\textbf{Matching Lower Bound or State-of-the-Art Results.} We continue to assume sufficiently large $T$. For regularized participation (Corollary~\ref{corollary:regularizedGiven}), our convergence rate of $\mathcal{O}\left(\sfrac{\sigma}{\sqrt{SIT}}\right)$ \textit{matches the lower bound of convergence error for centralized SGD} \cite{arjevani2019lower}. The lower bound states that to reach a convergence error of $\epsilon$, there exists an objective function such that at least $\Omega\left(\sfrac{\sigma^2}{\epsilon^2}\right)$ stochastic gradient oracle calls are required.\footnote{Note that \cite{arjevani2019lower} defines the convergence error as $\norm{\nabla f(\x_{{t}})}$ whereas we define it as $\normsq{\nabla f(\x_{{t}})}$. Hence, the order of our $\epsilon$ is different from that in \cite{arjevani2019lower}.} We have $SIT$ oracle calls to reach $\epsilon = \mathcal{O}\left(\sfrac{\sigma}{\sqrt{SIT}}\right)$ according to our result, thus our upper bound matches this lower bound and is asymptotically optimal. Note that our upper bound still matches the lower bound if we include the coefficient $\sqrt{L\mathcal{F}}$ in  $\mathcal{O}\left(\cdot\right)$ and $L\mathcal{F}$ in $\Omega\left(\cdot\right)$. For stochastic participation (Corollary~\ref{corollary:nonRegularizedExpectation}), our convergence rate of $\mathcal{O}\left(\sfrac{(1+\sigma^2)}{\sqrt{SIT}}\right)$ \textit{matches state-of-the-art results of FedAvg} \cite{fraboni2021clustered,karimireddy2020scaffold,yang2021achieving} that were obtained for the more idealized case of clients being selected to participate according to a specific independent sampling scheme. 

\textbf{Special Case of Waiting for All Clients.} 
We consider a specific configuration of FedAvg where all the clients wait for $P$ rounds before proceeding with SGD. In every iteration $i$ of round $t$ within these $P$ rounds, each participating client $n$ may call the stochastic gradient oracle to obtain $\g_n(\x_t)$, but it does not perform local updates. Instead, it accumulates and averages the sampled instances of $\g_n(\x_t)$. SGD is performed once after these $P$ rounds using the \textit{average} $\g_n(\x_t)$ from each client $n$. Assume that $P$ is large enough so that each client $n$ computes $\g_n(\x_t)$ at least $M\geq 1$ times that are averaged afterwards. If the $M$ instances of $\g_n(\x_t)$ are all independent from each other, the noise $\sigma^2$ (defined in Assumption~\ref{assumption:gradientNoise}) becomes $\sfrac{\sigma^2}{M}$. 
This configuration can be considered as a special case of Algorithm~\ref{alg:main-alg}, where all $N$ clients perform SGD once in $P$ rounds, so there are $\sfrac{T}{P}$ updates in total, with a reduced stochastic gradient noise of $\sfrac{\sigma^2}{M}$. By replacing $\sigma$ and $SIT$ in our bound $\mathcal{O}\left(\sfrac{\sigma}{\sqrt{SIT}}\right)$ with $\sfrac{\sigma}{\sqrt{M}}$ and $\sfrac{NT}{P}$, respectively, we obtain a convergence rate of $\mathcal{O}\left(\sfrac{\sigma\sqrt{P}}{\sqrt{MNT}}\right)$ for regularized participation with this ``wait-for-all'' method.

In theory, when considering regularized participation and $M$ is large enough, the convergence rate of this ``wait-for-all'' method can potentially match with that of the original configuration of Algorithm~\ref{alg:main-alg}. In this way, the convergence error upper bound of both methods can match the lower bound of SGD convergence error (see discussion above). However, in practice, reducing the stochastic gradient noise as in this ``wait-for-all'' method may only give limited improvement. One reason is that practical settings have a finite number of training data samples, so the noise can only be reduced by a certain degree. In addition, having some noise in SGD can prevent the model parameter from being trapped in saddle points and local minima in practice. The empirical results in Section~\ref{sec:experiments} show that Algorithm~\ref{alg:main-alg} with amplification performs better than waiting for all the clients. We also note that our results for stochastic participation (Corollary~\ref{corollary:nonRegularizedExpectation}) \textit{does not} require all the clients to participate %
within $P$ rounds.

\textbf{Achieving Regularized Participation in Practice.}
In theory, regularized participation requires that $\frac{1}{\period } \sum_{t={t_0}}^{{t_0}+\period -1} q_t^n$ is equal to each other for all the clients $n\in\{1,\ldots,N\}$, for a properly chosen~$P$.
When all the clients are connected to the system and client sampling is performed to limit the computation and communication overhead, this condition holds when the clients are selected according to a permutation, i.e., all the clients participate once before the same client can be selected again. In the case where clients get disconnected from time to time, the system can try to properly schedule the participation of connected clients and their weights $\{q_t^n\}$, so that within a cycle of $\period$ rounds, all the clients participate with equal averaged weights. In practice, it can be sufficient if this regularized condition is only approximately satisfied, especially when multiple clients have similar datasets so that $\tilde{\delta}^2(\period)$ is small although not necessarily zero.
The empirical results in Section~\ref{sec:experiments} show that Algorithm~\ref{alg:main-alg} with amplification can give good performance even if $P$ is less than the participation cycle of all clients.
This suggests that, as long as the subset of clients that participate in $P$ rounds are more representative of the overall data distribution than the (usually much smaller) subset of clients that participate in a single round, amplifying the updates every $P$ rounds can be useful.
% \vspace{-0.5em}

\section{\texorpdfstring{Experiments}{Experiments}}
\label{sec:experiments}

We ran experiments of training convolutional neural networks (CNNs) with FashionMNIST~\cite{FashionMNIST} and CIFAR-10~\cite{CIFAR10} datasets, each of which has images in $10$ different classes.  We set the total number of clients to $N=250$. 
Similar to existing works~\cite{yang2021achieving,ding2020distributed,eichner2019semi}, we partition the data into clients so that each client has data of one majority class label, with $5\%$ of data of other (minority) labels, to simulate a setup with non-IID data distribution which is often encountered in FL scenarios.
We assume that the clients' availability exhibits a periodic pattern inspired by~\cite{ding2020distributed,eichner2019semi}. Namely, in the first $100$ rounds, only clients with the first two majority labels are available; in the next $100$ rounds, only clients with the next two majority labels are available, and so on. In each round, $S=10$ clients participate in a regularized manner, out of all the currently available clients.
To speed up initial rounds of training, similar to practical deep learning implementations, we start with standard FedAvg with a relatively large initial learning rate. The initial rates are $\gamma=0.1$ and $\gamma=0.05$ without amplification (i.e., $\eta=1$) for FashionMNIST and CIFAR-10, respectively, which were obtained using grid search in a separate scenario of always participation. 
After an initial training of $2,000$ rounds for FashionMNIST and $4,000$ rounds for CIFAR-10, we study the performance of different approaches with their own learning rates.  
Additional setup details and results are given in Appendix~\ref{sec:appendixExperiment}.

\textbf{Comparing Different Methods.}
We show the results of Algorithm~\ref{alg:main-alg} both with and without amplification. When using amplification, we set $\eta=10$ and $P=500$.
We also compare to an algorithm that waits for all the clients,  related to our discussion in Section~\ref{sec:discussions}, for two settings where each client computes its gradients either on a minibatch or on the entire dataset, which we refer to as ``wait-minibatch'' and ``wait-full'', respectively.
The best learning rate $\gamma$ of each approach was separately found on a grid that is $\{1, 0.1, 0.01, 0.001, 0.0001\}$ times the initial learning rate, with the periodic participation pattern described above.
The results are shown in Figure~\ref{fig:periodicAvailable}.

\begin{figure*}[b]
\includegraphics[width=\linewidth]{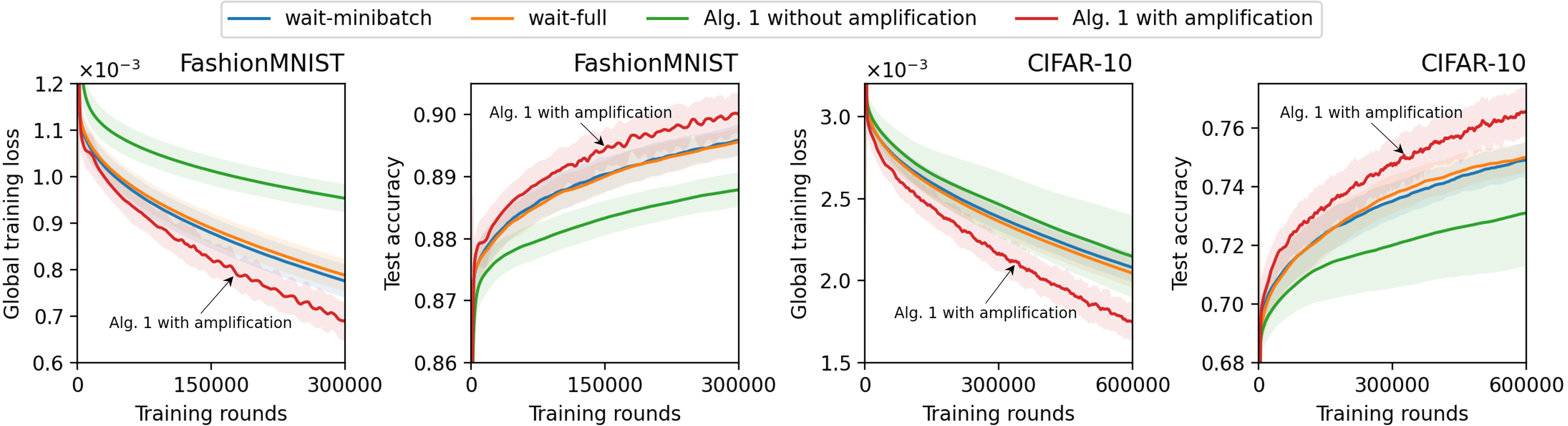}%
\vspace{-0.5em}
\caption{Results for different approaches with periodically connected clients ($\period=500$).
}
\label{fig:periodicAvailable}
\end{figure*}

We make a few key observations as follows.
First, we clearly see that \textit{Algorithm~\ref{alg:main-alg} with amplification gives the best performance}. By choosing $P=500$, we match with the participation period of clients, as described in the setup above. This corresponds to our setting with regularized participation. In this case, Algorithm~\ref{alg:main-alg} with amplification allows different groups of clients to make small progress among themselves first. After $P$ rounds, each client has made its own share of contribution, and the updates get amplified so that the model parameter progresses faster towards the direction that is overall beneficial for all the clients. As our theory predicts, this approach gives a desirable convergence rate, which is now verified by experiments. Next, we observe that \textit{the ``wait-minibatch'' and ``wait-full'' methods perform worse than Algorithm~\ref{alg:main-alg} with amplification}.
We also observe that ``wait-full'', which computes the gradient on the entire dataset, does not provide substantial gain compared to ``wait-minibatch'' in the experiments. These observations align with our discussion in Section~\ref{sec:discussions} on the practical performance of different methods.
Finally, we observe that \textit{standard FedAvg (i.e., Algorithm~\ref{alg:main-alg} without amplification) gives the lowest performance}. This is because without amplification, the algorithm cannot put more emphasis on the collective updates by the cohort of all clients, and the parameter updates may diverge from the overall optimal direction.

\textbf{Different Values of $P$.}
Next, we fix $\eta=10$ in Algorithm~\ref{alg:main-alg} and study the effect of choosing different values of $P$. We use the same learning rate $\gamma$ that is obtained from grid search in the previous experiment.
Note that the clients' availability pattern remains the same as described above.
The choice of different $P$ simulates the practical scenario where the estimation of $P$ may not perfectly align with the actual participation cycle. The results are shown in Figure~\ref{fig:periodicAvailableDiffP}.

We observe that $P=1$, which corresponds to the ``classical'' setting of FedAvg with two-sided learning rates \cite{karimireddy2020scaffold,yang2021achieving}, does not give the best performance. Interestingly, compared to $P=500$, we see that $P=100$ and $P=300$ give a similar performance, and even slightly better performance in the case of CIFAR-10 data. 
Due to the random offset applied to the first participation cycle in each experiment (see Appendix~\ref{subsec:appendixExperimentSetup} for details), every interval of $100$ rounds can include the ``partial'' contributions by two subsets of clients (with data in $4$, out of $10$, majority classes). The results suggest that amplifying such partial contributions by multiple subsets of clients can still improve performance. In the case of CIFAR-10 data, the reason for $P=500$ being slightly worse may be that the accumulated update within $P$ rounds generally becomes larger as $P$ gets larger, in which case amplification causes a bigger change in the model parameter. Depending on the landscape of the loss function, this change may be too big so that the model performance decreases. For a similar reason, the performances of $P=700$ and $P=900$ are even worse. 
Nevertheless, we expect that choosing a smaller learning rate $\gamma$ can improve the performance for large $P$ values. 

\begin{figure}[tb]
\centering
\includegraphics[width=1\linewidth]{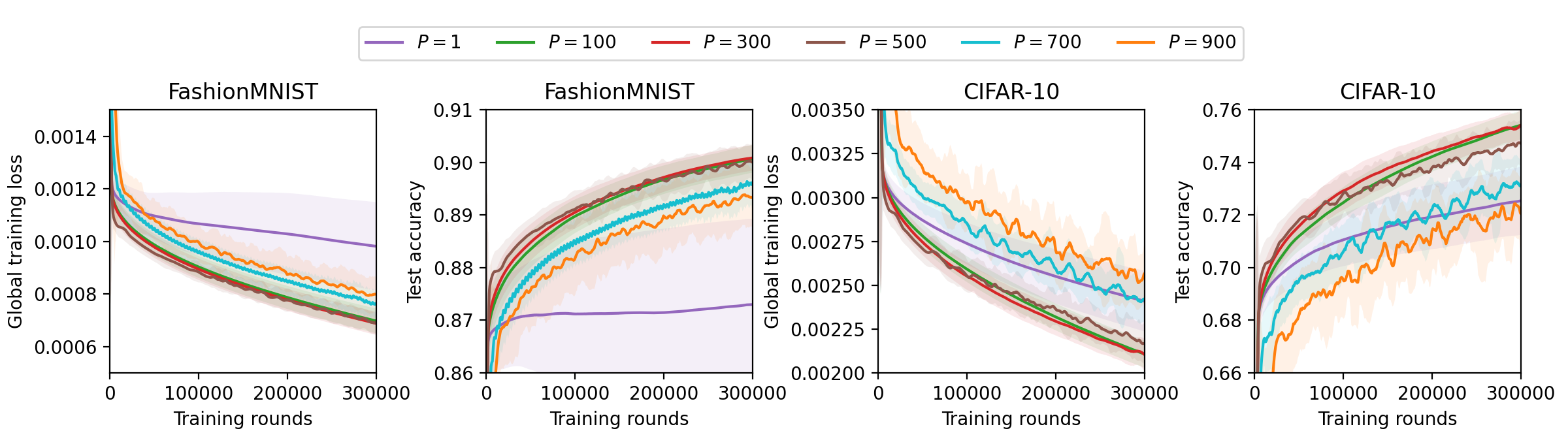}%
\vspace{-0.5em}
\caption{Algorithm~\ref{alg:main-alg} with amplification and different $P$, with periodically connected clients.
}
\label{fig:periodicAvailableDiffP}
\end{figure}

\section{Conclusion}
In this paper, we have studied FL with arbitrary client participation. 
For a generalized FedAvg algorithm that amplifies parameter updates every $\period$ rounds, we have developed a unified framework for convergence analysis and obtained convergence rates for a variety of client participation patterns. 
Our findings suggest that regularized participation with finite $\period$ gives the best performance and matches the lower bound of SGD convergence error when $T$ is sufficiently large. For stochastic participation, first, we have formally proven that convergence is guaranteed when the participation process is ergodic. Then, for two generic classes of participation processes, we have proven that with a properly chosen $\period$, our convergence rate matches state-of-the-art FedAvg convergence rates that were derived for the idealized case of independent and unbiased participation. The empirical results have confirmed that amplification is useful and also provided further insights.
Future directions include analysis of advanced FL algorithms and more detailed empirical study.

\clearpage

\bibliography{references}

\begin{thebibliography}{42}
\providecommand{\natexlab}[1]{#1}
\providecommand{\url}[1]{\texttt{#1}}
\expandafter\ifx\csname urlstyle\endcsname\relax
  \providecommand{\doi}[1]{doi: #1}\else
  \providecommand{\doi}{doi: \begingroup \urlstyle{rm}\Url}\fi

\bibitem[Arjevani et~al.(2022)Arjevani, Carmon, Duchi, Foster, Srebro, and
  Woodworth]{arjevani2019lower}
Y.~Arjevani, Y.~Carmon, J.~C. Duchi, D.~J. Foster, N.~Srebro, and B.~Woodworth.
\newblock Lower bounds for non-convex stochastic optimization.
\newblock \emph{Mathematical Programming}, 2022.

\bibitem[Billingsley(1995)]{billingsley1995probability}
P.~Billingsley.
\newblock \emph{Probability and measure}.
\newblock John Wiley \& Sons, 1995.

\bibitem[Bonawitz et~al.(2019)Bonawitz, Eichner, Grieskamp, Huba, Ingerman,
  Ivanov, Kiddon, Kone\v{c}n\'{y}, Mazzocchi, McMahan, Van~Overveldt, Petrou,
  Ramage, and Roselander]{Bonawitz2019Towards}
K.~Bonawitz, H.~Eichner, W.~Grieskamp, D.~Huba, A.~Ingerman, V.~Ivanov,
  C.~Kiddon, J.~Kone\v{c}n\'{y}, S.~Mazzocchi, B.~McMahan, T.~Van~Overveldt,
  D.~Petrou, D.~Ramage, and J.~Roselander.
\newblock Towards federated learning at scale: System design.
\newblock In \emph{Proceedings of Machine Learning and Systems}, volume~1,
  pages 374--388, 2019.

\bibitem[Chen et~al.(2020)Chen, Horvath, and Richtarik]{chen2020optimal}
W.~Chen, S.~Horvath, and P.~Richtarik.
\newblock Optimal client sampling for federated learning.
\newblock \emph{arXiv preprint arXiv:2010.13723}, 2020.

\bibitem[Cho et~al.(2022)Cho, Wang, and Joshi]{cho2022towards}
Y.~J. Cho, J.~Wang, and G.~Joshi.
\newblock Towards understanding biased client selection in federated learning.
\newblock In \emph{International Conference on Artificial Intelligence and
  Statistics}, pages 10351--10375. PMLR, 2022.

\bibitem[Ding et~al.(2020)Ding, Niu, Yan, Zheng, Wu, Chen, Tang, and
  Jia]{ding2020distributed}
Y.~Ding, C.~Niu, Y.~Yan, Z.~Zheng, F.~Wu, G.~Chen, S.~Tang, and R.~Jia.
\newblock Distributed optimization over block-cyclic data.
\newblock \emph{arXiv preprint arXiv:2002.07454}, 2020.

\bibitem[Eichner et~al.(2019)Eichner, Koren, McMahan, Srebro, and
  Talwar]{eichner2019semi}
H.~Eichner, T.~Koren, B.~McMahan, N.~Srebro, and K.~Talwar.
\newblock Semi-cyclic stochastic gradient descent.
\newblock In \emph{International Conference on Machine Learning}, pages
  1764--1773. PMLR, 2019.

\bibitem[Fraboni et~al.(2021{\natexlab{a}})Fraboni, Vidal, Kameni, and
  Lorenzi]{fraboni2021clustered}
Y.~Fraboni, R.~Vidal, L.~Kameni, and M.~Lorenzi.
\newblock Clustered sampling: Low-variance and improved representativity for
  clients selection in federated learning.
\newblock In \emph{International Conference on Machine Learning}, volume 139,
  pages 3407--3416. PMLR, Jul. 2021{\natexlab{a}}.

\bibitem[Fraboni et~al.(2021{\natexlab{b}})Fraboni, Vidal, Kameni, and
  Lorenzi]{fraboni2021impact}
Y.~Fraboni, R.~Vidal, L.~Kameni, and M.~Lorenzi.
\newblock On the impact of client sampling on federated learning convergence.
\newblock \emph{arXiv preprint arXiv:2107.12211}, 2021{\natexlab{b}}.

\bibitem[Gorbunov et~al.(2021)Gorbunov, Hanzely, and
  Richtarik]{Gorbunov2021Local}
E.~Gorbunov, F.~Hanzely, and P.~Richtarik.
\newblock Local {SGD}: Unified theory and new efficient methods.
\newblock In \emph{International Conference on Artificial Intelligence and
  Statistics}, volume 130 of \emph{PMLR}, pages 3556--3564, 2021.

\bibitem[Gower et~al.(2019)Gower, Loizou, Qian, Sailanbayev, Shulgin, and
  Richt{\'a}rik]{gower2019sgd}
R.~M. Gower, N.~Loizou, X.~Qian, A.~Sailanbayev, E.~Shulgin, and
  P.~Richt{\'a}rik.
\newblock {SGD}: General analysis and improved rates.
\newblock In \emph{International Conference on Machine Learning}, pages
  5200--5209. PMLR, 2019.

\bibitem[Gu et~al.(2021)Gu, Huang, Zhang, and Huang]{gu2021fast}
X.~Gu, K.~Huang, J.~Zhang, and L.~Huang.
\newblock Fast federated learning in the presence of arbitrary device
  unavailability.
\newblock In \emph{Advances in Neural Information Processing Systems}, 2021.

\bibitem[G{\"u}rb{\"u}zbalaban et~al.(2021)G{\"u}rb{\"u}zbalaban, Ozdaglar, and
  Parrilo]{gurbuzbalaban2021random}
M.~G{\"u}rb{\"u}zbalaban, A.~Ozdaglar, and P.~A. Parrilo.
\newblock Why random reshuffling beats stochastic gradient descent.
\newblock \emph{Mathematical Programming}, 186\penalty0 (1):\penalty0 49--84,
  2021.

\bibitem[Haddadpour et~al.(2019)Haddadpour, Kamani, Mahdavi, and
  Cadambe]{Haddadpour2019Local}
F.~Haddadpour, M.~M. Kamani, M.~Mahdavi, and V.~Cadambe.
\newblock Local {SGD} with periodic averaging: Tighter analysis and adaptive
  synchronization.
\newblock In \emph{Advances in Neural Information Processing Systems}, 2019.

\bibitem[He et~al.(2015)He, Zhang, Ren, and Sun]{he2015delving}
K.~He, X.~Zhang, S.~Ren, and J.~Sun.
\newblock Delving deep into rectifiers: Surpassing human-level performance on
  imagenet classification.
\newblock In \emph{IEEE International Conference on Computer Vision}, pages
  1026--1034, 2015.

\bibitem[Hoeffding(1994)]{hoeffding1994probability}
W.~Hoeffding.
\newblock Probability inequalities for sums of bounded random variables.
\newblock In \emph{The collected works of Wassily Hoeffding}, pages 409--426.
  Springer, 1994.

\bibitem[Kairouz et~al.(2021)Kairouz, McMahan, Avent, Bellet, Bennis, Bhagoji,
  Bonawitz, Charles, Cormode, Cummings, et~al.]{kairouz2019advances}
P.~Kairouz, H.~B. McMahan, B.~Avent, A.~Bellet, M.~Bennis, A.~N. Bhagoji,
  K.~Bonawitz, Z.~Charles, G.~Cormode, R.~Cummings, et~al.
\newblock Advances and open problems in federated learning.
\newblock \emph{Foundations and Trends{\textregistered} in Machine Learning},
  14\penalty0 (1--2):\penalty0 1--210, 2021.

\bibitem[Karimireddy et~al.(2020)Karimireddy, Kale, Mohri, Reddi, Stich, and
  Suresh]{karimireddy2020scaffold}
S.~P. Karimireddy, S.~Kale, M.~Mohri, S.~Reddi, S.~Stich, and A.~T. Suresh.
\newblock {SCAFFOLD}: Stochastic controlled averaging for federated learning.
\newblock In \emph{International Conference on Machine Learning}, pages
  5132--5143. PMLR, 2020.

\bibitem[Krizhevsky and Hinton(2009)]{CIFAR10}
A.~Krizhevsky and G.~Hinton.
\newblock Learning multiple layers of features from tiny images.
\newblock Technical report, University of Toronto, 2009.

\bibitem[Li et~al.(2020{\natexlab{a}})Li, Sahu, Talwalkar, and
  Smith]{li2020federated}
T.~Li, A.~K. Sahu, A.~Talwalkar, and V.~Smith.
\newblock Federated learning: Challenges, methods, and future directions.
\newblock \emph{IEEE Signal Processing Magazine}, 37\penalty0 (3):\penalty0
  50--60, 2020{\natexlab{a}}.

\bibitem[Li et~al.(2020{\natexlab{b}})Li, Sahu, Zaheer, Sanjabi, Talwalkar, and
  Smith]{li2020federatedOptimization}
T.~Li, A.~K. Sahu, M.~Zaheer, M.~Sanjabi, A.~Talwalkar, and V.~Smith.
\newblock Federated optimization in heterogeneous networks.
\newblock \emph{Proceedings of Machine Learning and Systems}, 2:\penalty0
  429--450, 2020{\natexlab{b}}.

\bibitem[Li et~al.(2020{\natexlab{c}})Li, Huang, Yang, Wang, and
  Zhang]{Li2020On}
X.~Li, K.~Huang, W.~Yang, S.~Wang, and Z.~Zhang.
\newblock On the convergence of fedavg on non-iid data.
\newblock In \emph{International Conference on Learning Representations},
  2020{\natexlab{c}}.

\bibitem[Lin et~al.(2020)Lin, Stich, Patel, and Jaggi]{Lin2020Dont}
T.~Lin, S.~U. Stich, K.~K. Patel, and M.~Jaggi.
\newblock Don't use large mini-batches, use local {SGD}.
\newblock In \emph{International Conference on Learning Representations}, 2020.

\bibitem[McMahan et~al.(2017)McMahan, Moore, Ramage, Hampson, and
  y~Arcas]{mcmahan2017communication}
B.~McMahan, E.~Moore, D.~Ramage, S.~Hampson, and B.~A. y~Arcas.
\newblock Communication-efficient learning of deep networks from decentralized
  data.
\newblock In \emph{Artificial Intelligence and Statistics}, pages 1273--1282.
  PMLR, 2017.

\bibitem[Mishchenko et~al.(2020)Mishchenko, Khaled Ragab~Bayoumi, and
  Richt{\'a}rik]{mishchenko2020random}
K.~Mishchenko, A.~Khaled Ragab~Bayoumi, and P.~Richt{\'a}rik.
\newblock Random reshuffling: Simple analysis with vast improvements.
\newblock \emph{Advances in Neural Information Processing Systems}, 33, 2020.

\bibitem[Reddi et~al.(2021)Reddi, Charles, Zaheer, Garrett, Rush,
  Kone{\v{c}}n{\'y}, Kumar, and McMahan]{reddi2021adaptive}
S.~J. Reddi, Z.~Charles, M.~Zaheer, Z.~Garrett, K.~Rush, J.~Kone{\v{c}}n{\'y},
  S.~Kumar, and H.~B. McMahan.
\newblock Adaptive federated optimization.
\newblock In \emph{International Conference on Learning Representations}, 2021.

\bibitem[Ruan et~al.(2021)Ruan, Zhang, Liang, and Joe-Wong]{ruan2021towards}
Y.~Ruan, X.~Zhang, S.-C. Liang, and C.~Joe-Wong.
\newblock Towards flexible device participation in federated learning.
\newblock In \emph{International Conference on Artificial Intelligence and
  Statistics}, pages 3403--3411. PMLR, 2021.

\bibitem[Stich(2019)]{stich2018local}
S.~U. Stich.
\newblock Local {SGD} converges fast and communicates little.
\newblock In \emph{International Conference on Learning Representations}, 2019.

\bibitem[Wang and Joshi(2019)]{MLSYS2019Jianyu}
J.~Wang and G.~Joshi.
\newblock Adaptive communication strategies to achieve the best error-runtime
  trade-off in local-update {SGD}.
\newblock In \emph{Proceedings of Machine Learning and Systems}, volume~1,
  pages 212--229, 2019.

\bibitem[Wang and Joshi(2021)]{CooperativeSGD}
J.~Wang and G.~Joshi.
\newblock Cooperative {SGD}: A unified framework for the design and analysis of
  local-update {SGD} algorithms.
\newblock \emph{Journal of Machine Learning Research}, 22\penalty0
  (213):\penalty0 1--50, 2021.

\bibitem[Wang et~al.(2020)Wang, Tantia, Ballas, and Rabbat]{wang2020lookahead}
J.~Wang, V.~Tantia, N.~Ballas, and M.~Rabbat.
\newblock Lookahead converges to stationary points of smooth non-convex
  functions.
\newblock In \emph{IEEE International Conference on Acoustics, Speech and
  Signal Processing (ICASSP)}, pages 8604--8608, 2020.

\bibitem[Woodworth et~al.(2020{\natexlab{a}})Woodworth, Patel, Stich, Dai,
  Bullins, Mcmahan, Shamir, and Srebro]{woodworth2020local}
B.~Woodworth, K.~K. Patel, S.~Stich, Z.~Dai, B.~Bullins, B.~Mcmahan, O.~Shamir,
  and N.~Srebro.
\newblock Is local {SGD} better than minibatch {SGD}?
\newblock In \emph{International Conference on Machine Learning}, pages
  10334--10343. PMLR, 2020{\natexlab{a}}.

\bibitem[Woodworth et~al.(2020{\natexlab{b}})Woodworth, Patel, and
  Srebro]{woodworth2020minibatch}
B.~E. Woodworth, K.~K. Patel, and N.~Srebro.
\newblock Minibatch vs local sgd for heterogeneous distributed learning.
\newblock \emph{Advances in Neural Information Processing Systems},
  33:\penalty0 6281--6292, 2020{\natexlab{b}}.

\bibitem[Xiao et~al.(2017)Xiao, Rasul, and Vollgraf]{FashionMNIST}
H.~Xiao, K.~Rasul, and R.~Vollgraf.
\newblock {Fashion-MNIST}: a novel image dataset for benchmarking machine
  learning algorithms.
\newblock \emph{arXiv preprint arXiv:1708.07747}, 2017.

\bibitem[Yan et~al.(2020)Yan, Niu, Ding, Zheng, Wu, Chen, Tang, and
  Wu]{yan2020distributed}
Y.~Yan, C.~Niu, Y.~Ding, Z.~Zheng, F.~Wu, G.~Chen, S.~Tang, and Z.~Wu.
\newblock Distributed non-convex optimization with sublinear speedup under
  intermittent client availability.
\newblock \emph{arXiv preprint arXiv:2002.07399}, 2020.

\bibitem[Yang et~al.(2021)Yang, Fang, and Liu]{yang2021achieving}
H.~Yang, M.~Fang, and J.~Liu.
\newblock Achieving linear speedup with partial worker participation in
  non-{IID} federated learning.
\newblock In \emph{International Conference on Learning Representations}, 2021.

\bibitem[Yang et~al.(2022)Yang, Zhang, Khanduri, and Liu]{yang2021anarchic}
H.~Yang, X.~Zhang, P.~Khanduri, and J.~Liu.
\newblock Anarchic federated learning.
\newblock In \emph{International Conference on Machine Learning}, pages
  25331--25363. PMLR, 2022.

\bibitem[Yang et~al.(2019)Yang, Liu, Chen, and Tong]{yang2019federated}
Q.~Yang, Y.~Liu, T.~Chen, and Y.~Tong.
\newblock Federated machine learning: Concept and applications.
\newblock \emph{ACM Transactions on Intelligent Systems and Technology (TIST)},
  10\penalty0 (2):\penalty0 12, 2019.

\bibitem[Yu et~al.(2019{\natexlab{a}})Yu, Jin, and Yang]{yu2019linear}
H.~Yu, R.~Jin, and S.~Yang.
\newblock On the linear speedup analysis of communication efficient momentum
  {SGD} for distributed non-convex optimization.
\newblock In \emph{International Conference on Machine Learning}, volume~97 of
  \emph{PMLR}, pages 7184--7193, 2019{\natexlab{a}}.

\bibitem[Yu et~al.(2019{\natexlab{b}})Yu, Yang, and Zhu]{yu2019parallel}
H.~Yu, S.~Yang, and S.~Zhu.
\newblock Parallel restarted {SGD} with faster convergence and less
  communication: Demystifying why model averaging works for deep learning.
\newblock In \emph{AAAI Conference on Artificial Intelligence}, pages
  5693--5700, 2019{\natexlab{b}}.

\bibitem[Yun et~al.(2022)Yun, Rajput, and Sra]{yun2022minibatch}
C.~Yun, S.~Rajput, and S.~Sra.
\newblock Minibatch vs local {SGD} with shuffling: Tight convergence bounds and
  beyond.
\newblock In \emph{International Conference on Learning Representations}, 2022.

\bibitem[Zhang et~al.(2019)Zhang, Lucas, Ba, and Hinton]{zhang2019lookahead}
M.~Zhang, J.~Lucas, J.~Ba, and G.~E. Hinton.
\newblock Lookahead optimizer: k steps forward, 1 step back.
\newblock \emph{Advances in Neural Information Processing Systems}, 32, 2019.

\end{thebibliography}
\bibliographystyle{abbrvnat}

\clearpage

\begin{center}
\LARGE \textbf{Appendix}
\end{center}

\appendix

\numberwithin{equation}{section}
\counterwithin{figure}{section}
\counterwithin{algocf}{section}
\counterwithin{theorem}{section}

\counterwithin{table}{section}

\startcontents[sections]
\printcontents[sections]{l}{1}{\setcounter{tocdepth}{2}}

\clearpage

\section{Additional Related Works}
\label{sec:appendixAdditionalRelatedWork}

We complement Section~\ref{sec:introduction} by discussing a few additional related works as follows.

Client unavailability was also studied from the perspective of block-cyclic SGD~\cite{ding2020distributed,eichner2019semi}, which trains \textit{multiple} models each for a block (group) of clients with known identity and was analyzed for convex objectives. Different from those works, our goal is to train a \textit{single} model without requiring block structure, and we focus on non-convex objectives that frequently arise in modern problems involving neural networks.

In the literature of SGD in the centralized setting, the effect of different ways of minibatch sampling has also been studied, such as IID sampling from an arbitrary distribution \cite{gower2019sgd} and random reshuffling \cite{gurbuzbalaban2021random,mishchenko2020random}. However, the problem in FL is much more complex and needs to incorporate the diversity of different clients' local objectives. Moreover, we focus on a unified framework that is more general than the special cases of IID sampling and random reshuffling.

FedAvg with two-sided learning rates has been studied in the literature \cite{karimireddy2020scaffold,yang2021achieving}, where a global learning rate is applied to the updates at the end of each round, giving an improved convergence bound compared to using a single-sided learning rate. This is equivalent to our Algorithm~\ref{alg:main-alg} with $\period=1$ and the amplification factor $\eta$ being the global learning rate. We extend this setup by allowing any $\period\geq 1$. Moreover, most existing works including \cite{karimireddy2020scaffold,yang2021achieving} have only shown the benefit of two-sided learning rates in a theoretical setting; their experiments still uses a global learning rate of $1.0$, which is equivalent to the case of a single-sided learning rate. Although some recent works have used different global learning rates in experiments, the best setting often remains to be choosing $1.0$ as the global learning rate in the case of FedAvg \cite[Table 8]{reddi2021adaptive}. In contrast, we have verified the benefit of amplification in our experiments (Section~\ref{sec:experiments}), for scenarios with more challenging client participation patterns than existing works.
The idea of amplification is also related to look-ahead algorithms \cite{zhang2019lookahead,wang2020lookahead}, which consider a two-agent setting, e.g., a server and a single client. Different from these works, we consider the unique challenges in federated learning with multiple arbitrarily participating clients, where the data is non-IID across clients.

It is also worth pointing out that our discussion related to waiting for all clients in Section~\ref{sec:discussions} has some analogy to the comparison between local SGD and minibatch SGD \cite{woodworth2020local,woodworth2020minibatch,yun2022minibatch}. Even when not considering the effect of partial participation, these existing works have only identified some specific classes of convex objectives where local SGD can be shown to have a better theoretical convergence rate than minibatch SGD. An improved theoretical approach to capture the fact that local SGD often outperforms minibatch SGD in practice, for a wide range of problems and objective functions (especially non-convex objectives), remains an interesting future direction.

\section{Additional Discussions}

\subsection{Why Does Amplification Help?}
\label{sec:appendixMotivatingExample}

We give an illustrative example to explain why the use of amplification with $\eta > 1$ generally improves the performance, as shown in both our theory and experiments. Consider a simple setting with $N=3$ clients. Each client $n$ has a local objective defined as $F_n(\x) = \frac{1}{2}\normsq{\x - \z_n}$ for some constant vector $\z_n \in \mathbb{R}^m$. It is evident that the true optimal solution that minimizes the global objective $f(\x) = \frac{1}{6}\sum_{n=1}^3 \normsq{\x - \z_n}$ is $\x^* = \frac{1}{3}\sum_{n=1}^3 \z_n$, but we assume that this is not known to the system and we would like to solve this problem using Algorithm~\ref{alg:main-alg}. In this specific example, we consider a two-dimensional space (i.e., $m=2$) and define $\z_1 = (-1, 0), \z_2 = (1,0), \z_3=(0, \sqrt{3})$. We further assume that the three clients are available in a cyclic manner, so that in the first round, only client $n=1$ is available; in the second round, only client $n=2$ is available; in the third round, only client $n=3$ is available; then this cycle continues with client $n=1$ available in the fourth round, and so on. We plot the trajectory of how the model parameter $\x$ changes from the initial value $\x_0$ to $\x_{15}$ that is obtained after $T=15$ rounds. The initial parameter is set to $\x_0 = (1,2)$ in this example. The plots of the trajectories with different choices of local learning rate $\gamma$ and amplification factor $\eta$ are shown in Figures~\ref{fig:motivationalAmplification1}--\ref{fig:motivationalAmplification2}.

With the local objectives defined in this example, the local gradient can be directly computed as $\nabla F_n(\x) = \x - \z_n$. Therefore, in the case of no amplification (i.e., $\eta=1$), the solution variable $\x$ moves towards $\z_1$ in the first round when client $n=1$ is available; then, it moves towards $\z_2$ in the second round when client $n=2$ is available, and so on. As we see from Figure~\ref{fig:motivationalAmplification1}, this can lead to a pattern where $\x$ cycles around the optimal solution $\x^*$ but approaches $\x^*$ only very slowly, especially when $\gamma$ is large. In the case of small $\gamma$, this cyclic pattern is not apparent, but we can still see that $\x$ moves towards different local optimal points (i.e., $\z_n$ for different $n$) in different rounds, and the overall convergence speed is slow. This leads to the following problem: regardless of whether we choose a large or small $\gamma$, at the end of the $15$-th round, there is an apparent gap between $\x_{15}$ and~$\x^*$.

The advantage of amplification is that it increases the importance of the aggregated updates made by clients that participate in different rounds. In this specific example, a single client's update direction may be towards its own local optimum, which can be different from the direction towards the global optimum. However, the accumulated updates by all the three clients over every three rounds will be more likely towards the global optimum (or a neighborhood of it). By amplifying such accumulated updates after every three rounds (i.e., $P=3$), we can let the solution variable $\x$ move much faster towards the true (global) optimum. This intuition is confirmed by the trajectories shown in Figure~\ref{fig:motivationalAmplification2}. Here, we note that the product of $\gamma$ and $\eta$ in the corresponding sub-figures of Figures~\ref{fig:motivationalAmplification1} and \ref{fig:motivationalAmplification2} are equal. By keeping a small local learning rate $\gamma$ and amplifying the updates accumulated in every $P=3$ rounds, we can converge to the optimal value $\x^*$, as shown in Figure~\ref{fig:motivationalAmplification2}.

Note that in practice, $P$ does not need to perfectly align with the ``cycle'' of participation, as seen in Figure~\ref{fig:periodicAvailableDiffP} in Section~\ref{sec:experiments} of the main paper.

\begin{figure}[H]
    \centering
    \begin{subfigure}[b]{0.245\textwidth}
        \centering
        \includegraphics[width=\textwidth]{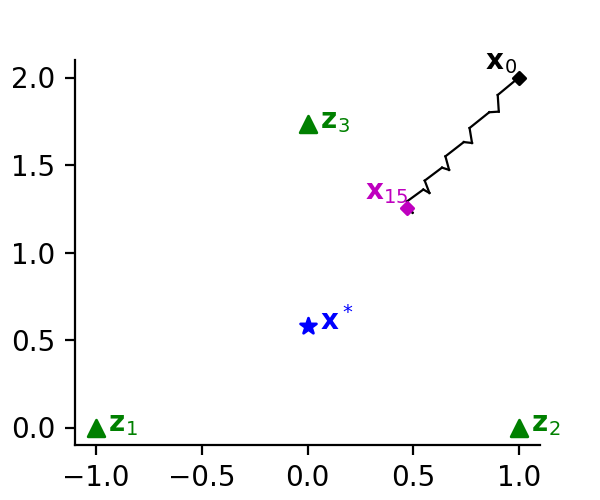}
        \caption{$\gamma=0.05, \eta=1$}
    \end{subfigure}
    \begin{subfigure}[b]{0.245\textwidth}
        \centering
        \includegraphics[width=\textwidth]{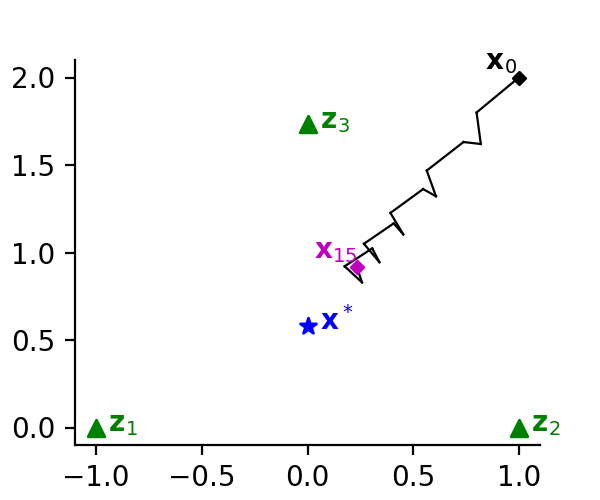}
        \caption{$\gamma=0.1, \eta=1$}
    \end{subfigure}
    \begin{subfigure}[b]{0.245\textwidth}
        \centering
        \includegraphics[width=\textwidth]{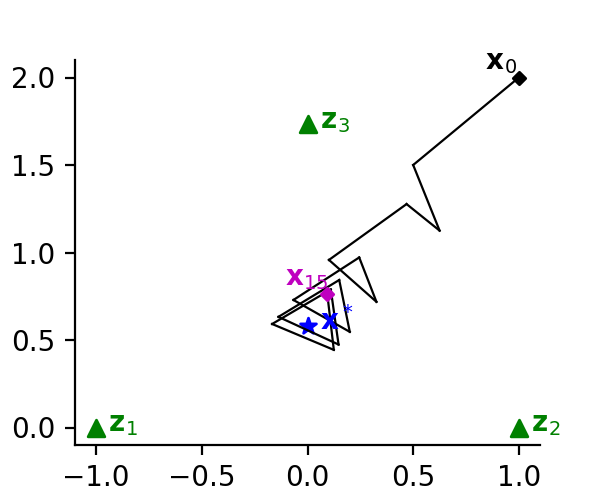}
        \caption{$\gamma=0.25, \eta=1$}
    \end{subfigure}
    \begin{subfigure}[b]{0.245\textwidth}
        \centering
        \includegraphics[width=\textwidth]{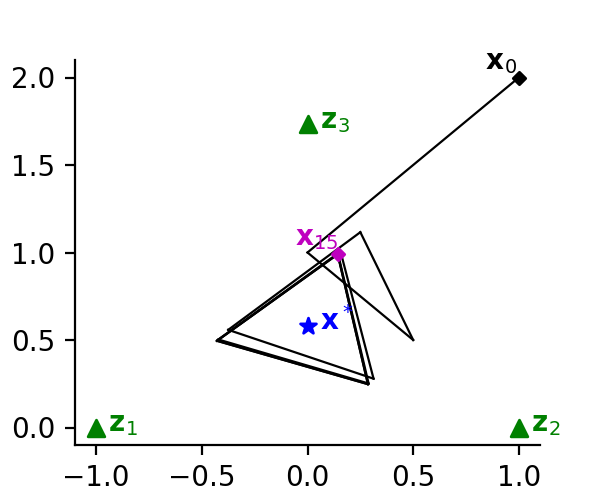}
        \caption{$\gamma=0.5, \eta=1$}
    \end{subfigure}    
    \caption{Motivating example: different local learning rates $\gamma$ without amplification (i.e., $\eta=1$). The trajectory from $\x_0$ to $\x_{15}$ shows how the model parameter changes from round $t=0$ to round $t=15$.}
    \label{fig:motivationalAmplification1}
\end{figure}
\begin{figure}[H]
    \centering
    \begin{subfigure}[b]{0.245\textwidth}
        \centering
        \includegraphics[width=\textwidth]{motivational_gamma0.05_eta1.png}
        \caption{$\gamma=0.05, \eta=1$}
    \end{subfigure}
    \begin{subfigure}[b]{0.245\textwidth}
        \centering
        \includegraphics[width=\textwidth]{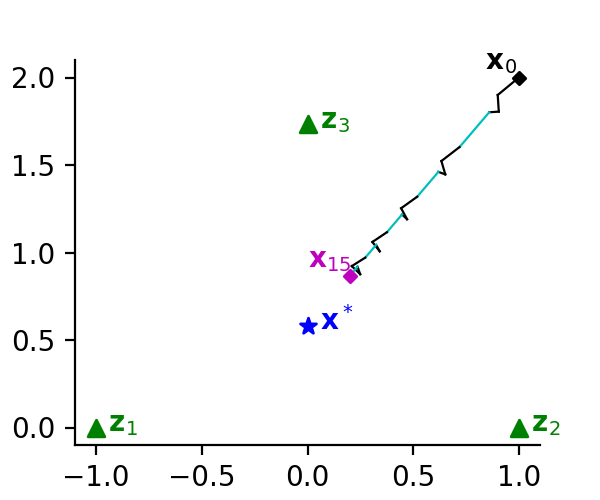}
        \caption{$\gamma=0.05, \eta=2$}
    \end{subfigure}
    \begin{subfigure}[b]{0.245\textwidth}
        \centering
        \includegraphics[width=\textwidth]{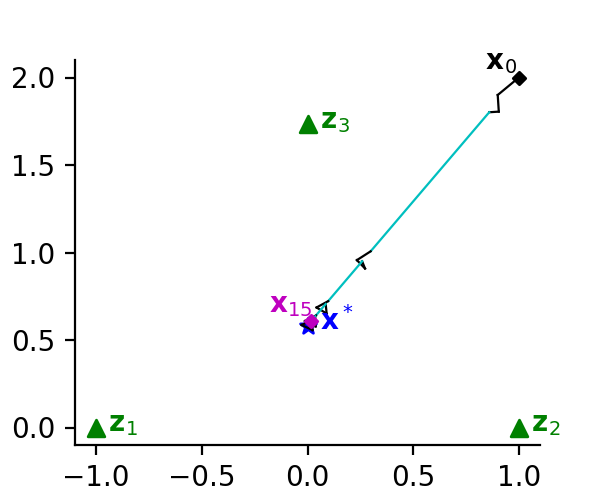}
        \caption{$\gamma=0.05, \eta=5$}
    \end{subfigure}
    \begin{subfigure}[b]{0.245\textwidth}
        \centering
        \includegraphics[width=\textwidth]{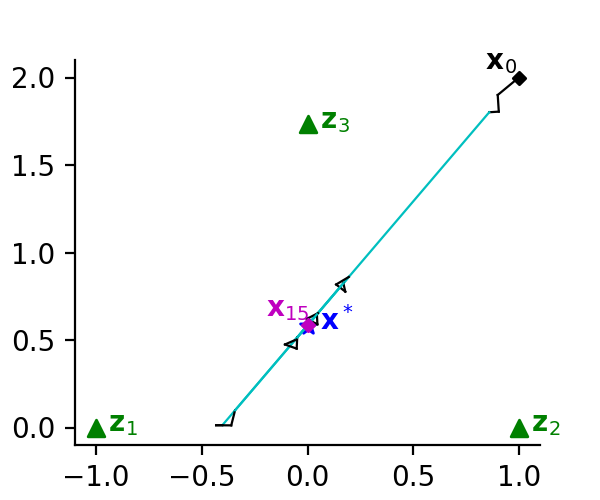}
        \caption{$\gamma=0.05, \eta=10$}
    \end{subfigure}    
    \caption{Motivating example: fixed local learning rate $\gamma=0.05$ with different amplification factors $\eta$. The trajectory from $\x_0$ to $\x_{15}$ shows how the model parameter changes from round $t=0$ to round $t=15$. The segments in \textcolor{cyan}{cyan} color shows the change in parameter $\x$ due to amplification, while the segments in black color shows the parameter change due to regular SGD operation.}
    \label{fig:motivationalAmplification2}
\end{figure}

\subsection{Relating Different Participation Patterns to Practical FL Scenarios}

In the following, we discuss some possible connections between practical FL application scenarios and the different participation patterns that are considered in our theoretical analysis.

The case of periodic shifts, which is often seen in practical cross-device FL scenarios, is mostly related to \textit{regularized participation}. In this case, the clients participate equally over the entire period, but only a subset of clients from a specific population (and with a specific data distribution) may participate in each ``period'' that includes a certain number of consecutive rounds. In practice, this ``equal participation'' only needs to hold approximately, and our empirical results at the end of Section~\ref{sec:experiments} show that our Algorithm~\ref{alg:main-alg} can still give good performance even if $P$ is less than the full cycle (period) of participation.

Among the stochastic participation patterns, \textit{ergodic} is the most generic. Intuitively, it says that the participation weights of different clients are the same when averaged over a long enough time. This can represent a cross-device FL scenario, where the participation of each client at any time instance can be highly random, but the long-term average statistics of clients are the same. The case of \textit{strongly mixing} participation includes the special case where the participation follows a Markov chain, i.e., if the client is currently unavailable, there is a certain (possibly small) probability that it will become available in the next round, and vice versa. This can also represent a cross-device FL setting, where the devices get connected and disconnected over time, and there is some randomness in the exact time it gets connected or disconnected. Finally, the case of \textit{independent} participation may represent a cross-silo FL setting, where the reason for not participating in a certain round is not because the client is disconnected for an extended period of time, but rather because the client has some other higher-priority tasks to run so that it cannot participate in all rounds.

\section{Proofs}
\label{sec:appendixProof}

\subsection{Preliminaries}

The following inequalities are frequently used throughout our proofs. We will use them without further explanation.

From Jensen's inequality, for any $\z_m \in \mathbb{R}^d, m\in\{1,2,\ldots,M\}$, we have
\begin{align}
    \normsq{ \frac{1}{M}\sum_{m=1}^M \z_m } &\leq \frac{1}{M}\sum_{m=1}^M \normsq{\z_m }
    \label{eq:Jensen1}
\end{align}
which directly gives
\begin{align}
    \normsq{ \sum_{m=1}^M \z_m } &\leq M \sum_{m=1}^M \normsq{ \z_m } .
    \label{eq:Jensen2}
\end{align}

Peter-Paul inequality (also known as Young's inequality) gives
\begin{equation}
    \langle \z_1, \z_2 \rangle \leq \frac{b \normsq{ \z_1 }}{2} + \frac{\normsq{ \z_2 }}{2b}
\end{equation}
for any $b > 0$ and any $\z_1, \z_2  \in \mathbb{R}^d$.
Hence, we also have
\begin{equation}
    \normsq{ \z_1 + \z_2 } \leq (1+b) \normsq{ \z_1 } + \left(1 +\frac{1}{b}\right)\normsq{ \z_2 }.
\end{equation}

Throughout the analysis, we use the short-hand notation $\Expectsubbracket{\cdot|\mathcal{Q}}{\cdot}$ to denote $\Expectbracket{\cdot|\mathcal{Q}}$.

In addition, we use the notations in Algorithm~\ref{alg:main-alg} and Assumptions~\ref{assumption:Lipschitz}, \ref{assumption:gradientNoise}, and \ref{assumption:gradientDivergenceAlternative}.

\subsection{Lemmas}
\begin{lemma}
\label{lemma:yDivergence}
When $\gamma \leq \frac{1}{12 LI\period }$, we have
\begin{align}
    &\sum_{n=1}^N q_t^n \Expectsubbracket{\cdot|\mathcal{Q}}{\normsq{\y^n_{t,i} - \sum_{n'=1}^N q_t^{n'}\y^{n'}_{t,i}}} \leq 3I\gamma^2\left( 6I \tilde{\nu}^2 +  \sigma^2 \left(1-\sum_{n=1}^N (q_t^n)^2\right)\right).
\end{align}
\end{lemma}
\begin{proof}
Note that
\begin{align}
    &\sum_{n=1}^N q_t^n\Expectsubbracket{\cdot|\mathcal{Q}}{\normsq{\y^n_{t,i+1} - \sum_{n'=1}^N q_t^{n'}\y^{n'}_{t,i+1}}} \nonumber \\*
    &= \sum_{n=1}^N q_t^n \Expectsubbracket{\cdot|\mathcal{Q}}{\normsq{\y^n_{t,i} - \gamma \g_n(\y^n_{t,i}) - \sum_{n'=1}^N q_t^{n'}(\y^{n'}_{t,i} - \gamma \g_{n'}(\y^{n'}_{t,i}))}}\nonumber \\
    &\overset{(a)}{=} \sum_{n=1}^N q_t^n \Expectsubbracket{\cdot|\mathcal{Q}}{\normsq{\y^n_{t,i} - \gamma \nabla F_n(\y^n_{t,i}) - \sum_{n'=1}^N q_t^{n'}(\y^{n'}_{t,i} - \gamma \nabla F_{n'}(\y^{n'}_{t,i}))}} \nonumber \\
    &\quad + \sum_{n=1}^N q_t^n \Expectsubbracket{\cdot|\mathcal{Q}}{\normsq{\gamma \left(\g_n(\y^n_{t,i}) - \nabla F_n(\y^n_{t,i})\right) - \gamma\sum_{n'=1}^N q_t^{n'} \left(\g_{n'}(\y^{n'}_{t,i}) - \nabla F_{n'}(\y^{n'}_{t,i}) \right)}}\nonumber \\
    &= \sum_{n=1}^N q_t^n \Expectsubbracket{\cdot|\mathcal{Q}}{\normsq{\y^n_{t,i} - \gamma \nabla F_n(\y^n_{t,i}) - \sum_{n'=1}^N q_t^{n'}(\y^{n'}_{t,i} - \gamma \nabla F_{n'}(\y^{n'}_{t,i}))}} \nonumber \\
    &\quad + \sum_{n=1}^N q_t^n \Expect_{\cdot|\mathcal{Q}}\Bigg[\Bigg\Vert - \gamma\sum_{n'\in\{1,\ldots,N\}\setminus \{n\}} q_t^{n'} \left(\g_{n'}(\y^{n'}_{t,i}) - \nabla F_{n'}(\y^{n'}_{t,i}) \right) \nonumber \\
    &\quad\quad\quad\quad\quad\quad\quad\quad  - \gamma \left(q_t^{n} - 1\right)\left(\g_n(\y^n_{t,i}) - \nabla F_n(\y^n_{t,i})\right) \Bigg\Vert^2\Bigg]\nonumber \\
    &\overset{(b)}{\leq} \sum_{n=1}^N q_t^n \Expectsubbracket{\cdot|\mathcal{Q}}{\normsq{\y^n_{t,i} - \gamma \nabla F_n(\y^n_{t,i}) - \sum_{n'=1}^N q_t^{n'}(\y^{n'}_{t,i} - \gamma \nabla F_{n'}(\y^{n'}_{t,i}))}} + \gamma^2 \sigma^2 \!\left(1\!-\!\sum_{n=1}^N (q_t^n)^2\right)\nonumber \\
    &= \sum_{n=1}^N q_t^n \Expect_{\cdot|\mathcal{Q}}\left[\left\Vert\y^n_{t,i} - \sum_{n'=1}^N q_t^{n'}\y^{n'}_{t,i} - \gamma \left[\nabla F_n(\y^n_{t,i}) - \nabla F_n\left(\sum_{n''=1}^N q_t^{n''}\y^{n''}_{t,i}\right) \right.\right.\right. \nonumber \\
    &\quad\quad\quad\quad\quad\quad\quad\quad + \nabla F_n\left(\sum_{n''=1}^N q_t^{n''}\y^{n''}_{t,i}\right) - \sum_{n'=1}^N q_t^{n'} \nabla F_{n'}\left(\sum_{n''=1}^N q_t^{n''}\y^{n''}_{t,i}\right)\nonumber \\
    &\quad\quad\quad\quad\quad\quad\quad\quad \left.\left.\left. + \sum_{n'=1}^N q_t^{n'} \nabla F_{n'}\left(\sum_{n''=1}^N q_t^{n''}\y^{n''}_{t,i}\right) - \sum_{n'=1}^N q_t^{n'} \nabla F_{n'}(\y^{n'}_{t,i})\right]\right\Vert^2\right] \nonumber \\
    &\quad+ \gamma^2 \sigma^2 \left(1-\sum_{n=1}^N (q_t^n)^2\right)\nonumber \\
    &\leq \left(1+\frac{1}{2I-1}\right)\sum_{n=1}^N q_t^n \Expectsubbracket{\cdot|\mathcal{Q}}{\normsq{\y^n_{t,i} - \sum_{n'=1}^N q_t^{n'}\y^{n'}_{t,i}}}\nonumber \\
    &\quad + 2I\gamma^2 \sum_{n=1}^N q_t^n\Expect_{\cdot|\mathcal{Q}}\left[\left\Vert\nabla F_n(\y^n_{t,i}) - \nabla F_n\left(\sum_{n''=1}^N q_t^{n''}\y^{n''}_{t,i}\right) + \nabla F_n\left(\sum_{n''=1}^N q_t^{n''}\y^{n''}_{t,i}\right) \right.\right. \nonumber \\
    &\quad\quad\quad\quad\quad\quad\quad\quad\quad\quad - \sum_{n'=1}^N q_t^{n'} \nabla F_{n'}\left(\sum_{n''=1}^N q_t^{n''}\y^{n''}_{t,i}\right) + \sum_{n'=1}^N q_t^{n'} \nabla F_{n'}\left(\sum_{n''=1}^N q_t^{n''}\y^{n''}_{t,i}\right)\nonumber \\
    &\quad\quad\quad\quad\quad\quad\quad\quad\quad\quad \left.\left.  - \sum_{n'=1}^N q_t^{n'} \nabla F_{n'}(\y^{n'}_{t,i})\right\Vert^2\right] \nonumber \\
    &\quad+ \gamma^2 \sigma^2 \left(1-\sum_{n=1}^N (q_t^n)^2\right)\nonumber \\
    &\leq \left(1+\frac{1}{2I-1}\right)\sum_{n=1}^N q_t^n \Expectsubbracket{\cdot|\mathcal{Q}}{\normsq{\y^n_{t,i} - \sum_{n'=1}^N q_t^{n'}\y^{n'}_{t,i}}}\nonumber \\
    &\quad + 6I\gamma^2 \sum_{n=1}^N q_t^n\Expectsubbracket{\cdot|\mathcal{Q}}{\normsq{\nabla F_n(\y^n_{t,i}) - \nabla F_n\left(\sum_{n''=1}^N q_t^{n''}\y^{n''}_{t,i}\right)}}\nonumber \\
    &\quad + 6I\gamma^2 \sum_{n=1}^N q_t^n\Expectsubbracket{\cdot|\mathcal{Q}}{\normsq{\nabla F_n\left(\sum_{n''=1}^N q_t^{n''}\y^{n''}_{t,i}\right) - \sum_{n'=1}^N q_t^{n'} \nabla F_{n'}\left(\sum_{n''=1}^N q_t^{n''}\y^{n''}_{t,i}\right)}}\nonumber \\
    &\quad + 6I\gamma^2 \sum_{n=1}^N q_t^n\Expectsubbracket{\cdot|\mathcal{Q}}{\normsq{\sum_{n'=1}^N q_t^{n'} \nabla F_{n'}\left(\sum_{n''=1}^N q_t^{n''}\y^{n''}_{t,i}\right) - \sum_{n'=1}^N q_t^{n'} \nabla F_{n'}(\y^{n'}_{t,i})}} \nonumber \\
    &\quad+ \gamma^2 \sigma^2 \left(1-\sum_{n=1}^N (q_t^n)^2\right)\nonumber \\
    &\overset{(c)}{\leq} \left(1+\frac{1}{2I-1}\right)\sum_{n=1}^N q_t^n \Expectsubbracket{\cdot|\mathcal{Q}}{\normsq{\y^n_{t,i} - \sum_{n'=1}^N q_t^{n'}\y^{n'}_{t,i}}} \nonumber \\
    &\quad + 6IL^2\gamma^2 \sum_{n=1}^N q_t^n\Expectsubbracket{\cdot|\mathcal{Q}}{\normsq{\y^n_{t,i} - \sum_{n''=1}^N q_t^{n''}\y^{n''}_{t,i}}} + 6I\gamma^2 \tilde{\nu}^2 \nonumber \\
    &\quad + 6IL^2\gamma^2 \sum_{n=1}^N q_t^n \sum_{n'=1}^N q_t^{n'}\Expectsubbracket{\cdot|\mathcal{Q}}{\normsq{ \y^{n'}_{t,i} - \sum_{n''=1}^N q_t^{n''}\y^{n''}_{t,i} }} + \gamma^2 \sigma^2 \left(1-\sum_{n=1}^N (q_t^n)^2\right) \nonumber \\
    &\overset{(d)}{=} \left(1+\frac{1}{2I-1} + 12IL^2\gamma^2\right)\sum_{n=1}^N q_t^n \Expectsubbracket{\cdot|\mathcal{Q}}{\normsq{\y^n_{t,i} - \sum_{n'=1}^N q_t^{n'}\y^{n'}_{t,i}}} \nonumber \\
    &\quad + 6I\gamma^2 \tilde{\nu}^2 + \gamma^2 \sigma^2 \left(1-\sum_{n=1}^N (q_t^n)^2\right) \nonumber \\
    &\overset{(e)}{\leq} \left(1+\frac{7}{12\left(I-\frac{1}{2}\right)}\right)\sum_{n=1}^N q_t^n \Expectsubbracket{\cdot|\mathcal{Q}}{\normsq{\y^n_{t,i} - \sum_{n'=1}^N q_t^{n'}\y^{n'}_{t,i}}} + 6I\gamma^2 \tilde{\nu}^2 + \gamma^2 \sigma^2 \left(1-\sum_{n=1}^N (q_t^n)^2\right).
    \label{eq:proofyDivergenceRecursion}
\end{align}
In the above, $(a)$ is due to $\Expectsubbracket{\cdot|\mathcal{Q}}{\normsq{\z}} = \normsq{\Expectsubbracket{\cdot|\mathcal{Q}}{\z}} + \Expectsubbracket{\cdot|\mathcal{Q}}{\normsq{\z - \Expectsubbracket{\cdot|\mathcal{Q}}{\z}}}$; $(b)$ is because the stochastic gradient noise is independent across clients, the variance of the sum of independent random variables is equal to the sum of the variance and $\mathrm{Var}(a\z) = a^2 \mathrm{Var}(\z)$ where $\mathrm{Var}(\z) := \Expectsubbracket{\cdot|\mathcal{Q}}{\normsq{\z - \Expectsubbracket{\cdot|\mathcal{Q}}{\z}}}$, and we also use the definition of $\sigma^2$ in Assumption~\ref{assumption:gradientNoise} while noting the law of total expectation as well as
\begin{align*}
    &\sum_{n=1}^N q_t^n \left[\sum_{n'\in\{1,\ldots,N\}\setminus \{n\} } (q_t^{n'})^2 + (1-q_t^n)^2\right] =\sum_{n=1}^N q_t^n \left[\sum_{n'=1}^N (q_t^{n'})^2 - (q_t^n)^2 + (1-q_t^n)^2\right] \\
    &= \sum_{n=1}^N (q_t^n)^2 + \sum_{n=1}^N q_t^n (1-2q_t^n) = \sum_{n=1}^N (q_t^n)^2 + \sum_{n=1}^N q_t^n - 2\sum_{n=1}^N (q_t^n)^2 = 1-\sum_{n=1}^N (q_t^n)^2;
\end{align*}
$(c)$ uses smoothness in the second term, the definition of $\tilde{\nu}$ in the third term, and Jensen's inequality followed by smoothness in the fourth term, while noting that $\sum_{n=1}^N q_t^n = 1$;
$(d)$ combines the second and fourth terms in the above line, while (again) noting that $\sum_{n=1}^N q_t^n = 1$; $(e)$ is because
\begin{align*}
    \frac{1}{2I-1} + 12IL^2\gamma^2 &\leq \frac{1}{2I-1} + \frac{1}{12I} \leq \frac{1}{2\left(I-\frac{1}{2}\right)} + \frac{1}{12\left(I-\frac{1}{2}\right)} = \frac{7}{12\left(I-\frac{1}{2}\right)}
\end{align*}
where the first inequality is due to $\gamma \leq \frac{1}{12 LI\period }  \leq \frac{1}{12 LI}$ since $\period \geq 1$.

For $i=0$, we note that
\begin{align}
    \sum_{n=1}^N q_t^n \Expectsubbracket{\cdot|\mathcal{Q}}{\normsq{\y^n_{t,0} - \sum_{n'=1}^N q_t^{n'}\y^{n'}_{t,0}}} = 0
    \label{eq:proofyDivergenceInitial}
\end{align}
because $\y^n_{t,0} = \x_t$ for all $n$ as specified in Algorithm~\ref{alg:main-alg}, and $\sum_{n=1}^N q_t^n=1$.

By combining \eqref{eq:proofyDivergenceRecursion} and \eqref{eq:proofyDivergenceInitial} and unrolling the recursion, we have for any $i$ that
\begin{align*}
    &\sum_{n=1}^N q_t^n \Expectsubbracket{\cdot|\mathcal{Q}}{\normsq{\y^n_{t,i} - \sum_{n'=1}^N q_t^{n'}\y^{n'}_{t,i}}}\\
    &\leq \sum_{i=0}^{I-1}\left(1+\frac{7}{12\left(I-\frac{1}{2}\right)}\right)^i \left(6I\gamma^2 \tilde{\nu}^2 + \gamma^2 \sigma^2 \left(1-\sum_{n=1}^N (q_t^n)^2\right)\right)\\
    &\overset{(a)}{=} \left[\left(1+\frac{7}{12\left(I-\frac{1}{2}\right)}\right)^I-1\right]\cdot \frac{12\left(I-\frac{1}{2}\right)}{7} \cdot\left( 6I\gamma^2 \tilde{\nu}^2 + \gamma^2 \sigma^2 \left(1-\sum_{n=1}^N (q_t^n)^2\right)\right)\\
    &= \left[\!\left(\! 1\!+\!\frac{7}{12\!\left(I\!-\!\frac{1}{2}\right)}\right)^{\!\!I-\frac{1}{2}}\!\left(\!1\!+\!\frac{7}{12\!\left(I\!-\!\frac{1}{2}\right)}\right)^{\!\!\frac{1}{2}}\!\!-\!1\right]\!\cdot\! \frac{12\left(I\!-\!\frac{1}{2}\right)}{7} \!\cdot\!\left( \!6I\gamma^2 \tilde{\nu}^2 \!+\! \gamma^2 \sigma^2 \left(1\!-\!\sum_{n=1}^N (q_t^n)^2\right)\!\right)\\
    &\overset{(b)}{\leq} \left[ e^{\frac{7}{12}} \cdot \sqrt{\frac{13}{6}}-1\right]\cdot \frac{12\left(I-\frac{1}{2}\right)}{7} \cdot\left( 6I\gamma^2 \tilde{\nu}^2 + \gamma^2 \sigma^2 \left(1-\sum_{n=1}^N (q_t^n)^2\right)\right)\\
    &\overset{(c)}{\leq} 3I\gamma^2\left( 6I \tilde{\nu}^2 +  \sigma^2 \left(1-\sum_{n=1}^N (q_t^n)^2\right)\right)
\end{align*}
where $(a)$ uses the expression of the sum of geometric progression series; $(b)$ uses $\left(1+x\right)^\frac{1}{x} \leq e$ and $I \geq 1$ hence $\left(1+\frac{7}{12\left(I-\frac{1}{2}\right)}\right)^{\frac{1}{2}} \leq \left(1+\frac{7}{6}\right)^{\frac{1}{2}} = \sqrt{\frac{13}{6}}$; $(c)$ is because $\left[ e^{\frac{7}{12}} \cdot \sqrt{\frac{13}{6}}-1\right]\frac{12}{7} \leq 3$ and $I-\frac{1}{2}\leq I$.
\end{proof}

\begin{lemma}
\label{lemma:overallChange}
When $\gamma \leq \frac{1}{12 LI\period }$, we have
\begin{align}
    &\Expectsubbracket{\cdot|\mathcal{Q}}{\normsq{\sum_{n=1}^N q_t^n\y^n_{t,i} - \x_{t_0}}} \nonumber \\
    &\leq 3\gamma^2 I^2 \tilde{\nu}^2  +   24\gamma^2 I^2 \period ^2  \tilde{\beta}^2  + 24 \gamma^2 I^2 \period ^2 \Expectsubbracket{\cdot|\mathcal{Q}}{\normsq{\nabla f(\x_{t_0})}} \nonumber\\
    &\quad+ 3 \gamma^2 I\period  \sigma^2 \left[ \rho^2   + \frac{1}{6\period }  \left(1-\sum_{n=1}^N (q_t^n)^2\right) \right]
\end{align}
for $t_0\leq t\leq t_0+\period-1$.
\end{lemma}
\begin{proof}

We have
\begin{align}
    & \Expectsubbracket{\cdot|\mathcal{Q}}{\normsq{\sum_{n=1}^N q_t^n\y^n_{t,i+1} - \x_{t_0}}}  \nonumber \\
    &= \Expectsubbracket{\cdot|\mathcal{Q}}{\normsq{\sum_{n=1}^N q_t^n\left(\y^n_{t,i} - \gamma \g_n(\y^n_{t,i})\right) - \x_{t_0}}}  \nonumber \\
    &= \Expectsubbracket{\cdot|\mathcal{Q}}{\normsq{\sum_{n=1}^N q_t^n\left(\y^n_{t,i} -  \gamma \nabla F_n(\y^n_{t,i})\right)  - \x_{t_0}}} + \Expectsubbracket{\cdot|\mathcal{Q}}{\normsq{\sum_{n=1}^N q_t^n\gamma\left( \g_n(\y^n_{t,i}) - \nabla F_n(\y^n_{t,i})\right)}}  \nonumber \\
    &\leq \Expect_{\cdot|\mathcal{Q}}\left[\left\Vert\sum_{n=1}^N q_t^n\left[\y^n_{t,i} - \gamma \left(\nabla F_n(\y^n_{t,i}) - \nabla F_n\left(\sum_{n'=1}^N q_t^{n'}\y^{n'}_{t,i}\right) + \nabla F_n\left(\sum_{n'=1}^N q_t^{n'}\y^{n'}_{t,i}\right) \right.\right.\right.\right. \nonumber \\
    &\quad\quad\quad\quad\quad\quad  - \nabla F_n(\x_{t_0}) + \nabla F_n(\x_{t_0}) - \nabla f(\x_{t_0}) + \nabla f(\x_{t_0}) \Bigg) \Bigg]  - \x_{t_0}\Bigg\Vert^2\Bigg] \nonumber + \gamma^2 \rho^2 \sigma^2    \nonumber \\
    &\leq \left(1+\frac{1}{2I\period -1}\right) \Expectsubbracket{\cdot|\mathcal{Q}}{\normsq{\sum_{n=1}^N q_t^n\y^n_{t,i}  - \x_{t_0}}} \nonumber \\
    &\quad + 8I\period \gamma^2 \Expectsubbracket{\cdot|\mathcal{Q}}{\normsq{\sum_{n=1}^N q_t^n\left[\nabla F_n(\y^n_{t,i}) - \nabla F_n\left(\sum_{n'=1}^N q_t^{n'}\y^{n'}_{t,i}\right)\right]}}  \nonumber \\
    &\quad + 8I\period \gamma^2 \Expectsubbracket{\cdot|\mathcal{Q}}{\normsq{\sum_{n=1}^N q_t^n\left[\nabla F_n\left(\sum_{n'=1}^N q_t^{n'}\y^{n'}_{t,i}\right) - \nabla F_n(\x_{t_0})\right]}} \nonumber \\
    &\quad + 8I\period \gamma^2 \Expectsubbracket{\cdot|\mathcal{Q}}{\normsq{\sum_{n=1}^N q_t^n\left[\nabla F_n(\x_{t_0}) - \nabla f(\x_{t_0})\right]}} \nonumber \\
    &\quad  + 8I\period \gamma^2 \Expectsubbracket{\cdot|\mathcal{Q}}{\normsq{\sum_{n=1}^N q_t^n \nabla f(\x_{t_0})}} + \gamma^2 \rho^2 \sigma^2   \nonumber \\
    &\overset{(a)}{\leq} \left(1+\frac{1}{2I\period -1}\right) \Expectsubbracket{\cdot|\mathcal{Q}}{\normsq{\sum_{n=1}^N q_t^n\y^n_{t,i}  - \x_{t_0}}} \nonumber \\
    &\quad + 8I\period \gamma^2 \sum_{n=1}^N q_t^n \Expectsubbracket{\cdot|\mathcal{Q}}{\normsq{\nabla F_n(\y^n_{t,i}) - \nabla F_n\left(\sum_{n'=1}^N q_t^{n'}\y^{n'}_{t,i}\right)}}  \nonumber \\
    &\quad + 8I\period \gamma^2 \sum_{n=1}^N q_t^n\Expectsubbracket{\cdot|\mathcal{Q}}{\normsq{\nabla F_n\left(\sum_{n'=1}^N q_t^{n'}\y^{n'}_{t,i}\right) - \nabla F_n(\x_{t_0})}} \nonumber \\
    &\quad + 8I\period \gamma^2 \tilde{\beta}^2  + 8I\period \gamma^2 \Expectsubbracket{\cdot|\mathcal{Q}}{\normsq{\nabla f(\x_{t_0})}} + \gamma^2 \rho^2 \sigma^2   \nonumber \\
    &\overset{(b)}{\leq} \left(1+\frac{1}{2I\period -1}\right) \Expectsubbracket{\cdot|\mathcal{Q}}{\normsq{\sum_{n=1}^N q_t^n\y^n_{t,i}  - \x_{t_0}}} \nonumber \\
    &\quad + 8I\period L^2\gamma^2 \sum_{n=1}^N q_t^n \Expectsubbracket{\cdot|\mathcal{Q}}{\normsq{\y^n_{t,i} - \sum_{n'=1}^N q_t^{n'}\y^{n'}_{t,i}}}  \nonumber \\
    &\quad + 8I\period L^2\gamma^2 \Expectsubbracket{\cdot|\mathcal{Q}}{\normsq{\sum_{n'=1}^N q_t^{n'}\y^{n'}_{t,i} - \x_{t_0}}} \nonumber \\
    &\quad + 8I\period \gamma^2 \tilde{\beta}^2 + 8I\period \gamma^2 \Expectsubbracket{\cdot|\mathcal{Q}}{\normsq{\nabla f(\x_{t_0})}} + \gamma^2 \rho^2 \sigma^2    \nonumber \\
    &\overset{(c)}{\leq} \left(1+\frac{1}{2I\period -1} + 8I\period L^2\gamma^2\right) \Expectsubbracket{\cdot|\mathcal{Q}}{\normsq{\sum_{n=1}^N q_t^n\y^n_{t,i}  - \x_{t_0}}} \nonumber \\
    &\quad + 24I^2\period L^2\gamma^4 \left( 6I \tilde{\nu}^2 +  \sigma^2 \left(1-\sum_{n=1}^N (q_t^n)^2\right)\right)     \nonumber \\
    &\quad +   8I\period \gamma^2 \tilde{\beta}^2  + 8I\period \gamma^2 \Expectsubbracket{\cdot|\mathcal{Q}}{\normsq{\nabla f(\x_{t_0})}} + \gamma^2 \rho^2 \sigma^2   \nonumber \\
    &\overset{(d)}{\leq} \left(1+\frac{5}{9\left(I\period -\frac{1}{2}\right)} \right) \Expectsubbracket{\cdot|\mathcal{Q}}{\normsq{\sum_{n=1}^N q_t^n\y^n_{t,i}  - \x_{t_0}}} +    144\gamma^4 I^3 \period  L^2 \tilde{\nu}^2  \nonumber \\
    &\quad +   8I\period \gamma^2 \tilde{\beta}^2  + 8I\period \gamma^2 \Expectsubbracket{\cdot|\mathcal{Q}}{\normsq{\nabla f(\x_{t_0})}} + \gamma^2 \sigma^2 \left[ \rho^2   + 24\gamma^2 I^2\period L^2  \left(1-\sum_{n=1}^N (q_t^n)^2\right) \right]\nonumber \\
    &\overset{(e)}{\leq} \left(1+\frac{5}{9\left(I\period -\frac{1}{2}\right)} \right) \Expectsubbracket{\cdot|\mathcal{Q}}{\normsq{\sum_{n=1}^N q_t^n\y^n_{t,i}  - \x_{t_0}}} +    \frac{\gamma^2 I \tilde{\nu}^2}{\period }  \nonumber \\
    &\quad +   8I\period \gamma^2 \tilde{\beta}^2  + 8I\period \gamma^2 \Expectsubbracket{\cdot|\mathcal{Q}}{\normsq{\nabla f(\x_{t_0})}} + \gamma^2 \sigma^2 \left[ \rho^2   + \frac{1}{6\period }  \left(1-\sum_{n=1}^N (q_t^n)^2\right) \right]
    .
    \label{eq:proofLemmaOverallChangeRecursion}
\end{align}

In the above, the stochastic gradient variance decomposition is the same as the proof of Lemma~\ref{lemma:yDivergence}, while noting that $\sum_{n=1}^N (q_t^n)^2 \leq \rho^2$; $(a)$ uses Jensen's inequality, the definition of $\tilde{\beta}^2$, and $\sum_{n=1}^N q_t^n =1$; $(b)$ uses $L$-smoothness; $(c)$ is from Lemma~\ref{lemma:yDivergence}; $(d)$ is due to
\begin{align*}
    \frac{1}{2I\period -1} + 8I\period L^2\gamma^2 &\leq \frac{1}{2I\period -1} + \frac{1}{18I\period } \leq \frac{1}{2\left(I\period -\frac{1}{2}\right)} + \frac{1}{18\left(I\period -\frac{1}{2}\right)} = \frac{5}{9\left(I\period -\frac{1}{2}\right)}
\end{align*}
where the first inequality is due to $\gamma \leq \frac{1}{12 LI\period }$; $(e)$ uses $\gamma \leq \frac{1}{12 LI\period }$ in the second and last terms.

We note that, for $t_0 \leq t \leq t_0 +\period-2$, we have
\begin{align}
    &\Expectsubbracket{\cdot|\mathcal{Q}}{\normsq{\sum_{n=1}^N q_{t+1}^n\y^n_{t+1,0} - \x_{t_0}}} \nonumber \\
    &= \Expectsubbracket{\cdot|\mathcal{Q}}{\normsq{\sum_{n=1}^N q_{t+1}^n\x_{t+1} - \x_{t_0}}} = \Expectsubbracket{\cdot|\mathcal{Q}}{\normsq{\x_{t+1} - \x_{t_0}}} = \Expectsubbracket{\cdot|\mathcal{Q}}{\normsq{\sum_{n=1}^N q_t^n\y^n_{t,I} - \x_{t_0}}}
    \label{eq:proofLemmaOverallChangeRecursionAcrosst1}
\end{align}
because $\y^n_{t,0} = \x_t$ for all $n$ as specified in Algorithm~\ref{alg:main-alg}, and $\sum_{n=1}^N q_t^n=1$ hence $\x_{t+1}=\x_t + \sum_{n=1}^N q_t^n (\y_{t,I}^n - \x_t) = \sum_{n=1}^N q_t^n \y_{t,I}^n$ according to Algorithm~\ref{alg:main-alg}.

Therefore, the same recursion in \eqref{eq:proofLemmaOverallChangeRecursion} also holds between $t+1$ and $t$, i.e.,
\begin{align}
    &\Expectsubbracket{\cdot|\mathcal{Q}}{\normsq{\sum_{n=1}^N q_{t+1}^n\y^n_{t+1,1} - \x_{t_0}}}  \nonumber \\*
    &\leq \left(1+\frac{5}{9\left(I\period -\frac{1}{2}\right)} \right) \Expectsubbracket{\cdot|\mathcal{Q}}{\normsq{\sum_{n=1}^N q_t^n\y^n_{t,I}  - \x_{t_0}}} \nonumber \\
    &\quad\quad  + \frac{\gamma^2 I \tilde{\nu}^2}{\period }  +   8I\period \gamma^2 \tilde{\beta}^2  + 8I\period \gamma^2 \Expectsubbracket{\cdot|\mathcal{Q}}{\normsq{\nabla f(\x_{t_0})}}+ \gamma^2 \sigma^2 \left[ \rho^2   + \frac{1}{6\period }  \left(1-\sum_{n=1}^N (q_t^n)^2\right) \right] .
    \label{eq:proofLemmaOverallChangeRecursionAcrosst2}
\end{align}

When $t=t_0$ and $i=0$, note that
\begin{align}
    \Expectsubbracket{\cdot|\mathcal{Q}}{\normsq{\sum_{n=1}^N q_{t_0}^n\y^n_{t_0,0} - \x_{t_0}}}=0
    \label{eq:proofLemmaOverallChangeInitial}
\end{align}
because $\y^n_{t_0,0} = \x_{t_0}$ for all $n$ as specified in Algorithm~\ref{alg:main-alg}, and $\sum_{n=1}^N q_t^n=1$.

We combine \eqref{eq:proofLemmaOverallChangeRecursion}, \eqref{eq:proofLemmaOverallChangeRecursionAcrosst2}, and \eqref{eq:proofLemmaOverallChangeInitial}. Since $t_0\leq t\leq t_0+\period-1$, unrolling the recursion gives
\begin{align*}
    &\Expectsubbracket{\cdot|\mathcal{Q}}{\normsq{\sum_{n=1}^N q_t^n\y^n_{t,i} - \x_{t_0}}} \\
    &\leq \sum_{\kappa = 0}^{I\period -1}\left(1+\frac{5}{9\left(I\period -\frac{1}{2}\right)} \right)^\kappa \Bigg(   \frac{\gamma^2 I \tilde{\nu}^2}{\period }  +   8I\period \gamma^2 \tilde{\beta}^2  + 8I\period \gamma^2 \Expectsubbracket{\cdot|\mathcal{Q}}{\normsq{\nabla f(\x_{t_0})}} \\
    &\quad\quad\quad\quad\quad\quad\quad\quad\quad\quad\quad\quad\quad\quad\quad\quad  + \gamma^2 \sigma^2 \left[ \rho^2   + \frac{1}{6\period }  \left(1-\sum_{n=1}^N (q_t^n)^2\right) \right] \Bigg) \\
    &= \left[\!\left(1+\frac{5}{9\left(I\period \!-\!\frac{1}{2}\right)} \right)^{I\period } \!\!\!- 1 \right] \!\cdot\frac{9\left(I\period \!-\!\frac{1}{2}\right)}{5}\cdot  \Bigg(   \frac{\gamma^2 I \tilde{\nu}^2}{\period }  +   8I\period \gamma^2 \tilde{\beta}^2  + 8I\period \gamma^2 \Expectsubbracket{\cdot|\mathcal{Q}}{\normsq{\nabla f(\x_{t_0})}} \\
    &\quad\quad\quad\quad\quad\quad\quad\quad\quad\quad\quad\quad\quad\quad\quad\quad+ \gamma^2 \sigma^2 \left[ \rho^2   + \frac{1}{6\period }  \left(1-\sum_{n=1}^N (q_t^n)^2\right) \right] \Bigg)\\
    &= \left[\!\left(1+\frac{5}{9\left(I\period \!-\!\frac{1}{2}\right)} \right)^{\!\!I\period -\frac{1}{2}} \left(1+\frac{5}{9\left(I\period -\frac{1}{2}\right)} \right)^{\!\!\frac{1}{2}} - 1 \right] \cdot\frac{9\left(I\period \!-\!\frac{1}{2}\right)}{5}\cdot  \Bigg(   \frac{\gamma^2 I \tilde{\nu}^2}{\period }  +   8I\period \gamma^2 \tilde{\beta}^2 \\
    &\quad\quad\quad\quad\quad\quad\quad\quad\quad\quad+ 8I\period \gamma^2 \Expectsubbracket{\cdot|\mathcal{Q}}{\normsq{\nabla f(\x_{t_0})}} + \gamma^2 \sigma^2 \left[ \rho^2   + \frac{1}{6\period }  \left(1-\sum_{n=1}^N (q_t^n)^2\right) \right] \Bigg)\\
    &\leq \left[e^{\frac{5}{9}}\!\cdot\!\sqrt{\frac{19}{9}}  - 1 \right] \cdot\frac{9\left(I\period \!-\!\frac{1}{2}\right)}{5}\cdot  \Bigg(   \frac{\gamma^2 I \tilde{\nu}^2}{\period }  +   8I\period \gamma^2 \tilde{\beta}^2  + 8I\period \gamma^2 \Expectsubbracket{\cdot|\mathcal{Q}}{\normsq{\nabla f(\x_{t_0})}} \\
    &\quad\quad\quad\quad\quad\quad\quad\quad\quad\quad\quad\quad\quad\quad\quad\quad + \gamma^2 \sigma^2 \left[ \rho^2   + \frac{1}{6\period }  \left(1\!-\!\sum_{n=1}^N (q_t^n)^2\right) \right] \Bigg)\\
    &\leq 3I\period  \left(   \frac{\gamma^2 I \tilde{\nu}^2}{\period }  +   8I\period \gamma^2 \tilde{\beta}^2  + 8I\period \gamma^2 \Expectsubbracket{\cdot|\mathcal{Q}}{\normsq{\nabla f(\x_{t_0})}} + \gamma^2 \sigma^2 \left[ \rho^2   + \frac{1}{6\period } \! \left(1\!-\!\sum_{n=1}^N (q_t^n)^2\right)\! \right]\! \right)\\
    &\leq 3\gamma^2 I^2 \tilde{\nu}^2  \!+ \!  24\gamma^2 I^2 \period ^2  \tilde{\beta}^2  \!+\! 24 \gamma^2 I^2 \period ^2 \Expectsubbracket{\cdot|\mathcal{Q}}{\normsq{\nabla f(\x_{t_0})}} \!+ \!3 \gamma^2 I\period  \sigma^2 \!\left[ \rho^2  \! + \!\frac{1}{6\period }  \!\left(\!1\!-\!\sum_{n=1}^N (q_t^n)^2\!\right) \!\right]
\end{align*}
where the rationale behind the steps is similar to the proof of Lemma~\ref{lemma:yDivergence}.
\end{proof}

\subsection{Proof of Theorem~\ref{theorem:mainConvergenceResult}}

Consider ${t_0}=k\period $ for $k=0,1,2,...$.

Due to smoothness,
\begin{align}
    &\Expectsubbracket{\cdot|\mathcal{Q},{t_0}}{f(\x_{{t_0}+\period })}\nonumber
    \\ 
    &\leq f(\x_{{t_0}}) +\Expectsubbracket{\cdot|\mathcal{Q},{t_0}}{\innerprod{\nabla f(\x_{{t_0}}), \x_{{t_0}+\period } -\x_{{t_0}} }} + \frac{L}{2}\Expectsubbracket{\cdot|\mathcal{Q},{t_0}}{\normsq{\x_{{t_0}+\period } -\x_{{t_0}}}}\nonumber \\
    &\leq f(\x_{{t_0}}) -\gamma \eta \innerprod{\nabla f(\x_{{t_0}}), \Expectsubbracket{\cdot|\mathcal{Q},{t_0}}{\sum_{t={t_0}}^{{t_0}+\period -1}\sum_{n=1}^N q_t^n \sum_{i=0}^{I-1}\g_n(\y^n_{t,i})}} \nonumber \\
    &\quad \quad + \frac{\gamma^2 \eta^2 L}{2}\Expectsubbracket{\cdot|\mathcal{Q},{t_0}}{\normsq{\sum_{t={t_0}}^{{t_0}+\period -1}\sum_{n=1}^N q_t^n \sum_{i=0}^{I-1}\g_n(\y^n_{t,i})}}\nonumber \\
    &\overset{(a)}{=} f(\x_{{t_0}}) -\gamma \eta \innerprod{\nabla f(\x_{{t_0}}), \Expectsubbracket{\cdot|\mathcal{Q},{t_0}}{\sum_{t={t_0}}^{{t_0}+\period -1}\sum_{n=1}^N q_t^n \sum_{i=0}^{I-1} \Expectbracket{\g_n(\y^n_{t,i}) \big| \mathcal{Q}, \y^n_{t,i}, \x_{{t_0}}}}} \nonumber\\
    &\quad \quad + \frac{\gamma^2\eta^2 L}{2}\Expectsubbracket{\cdot|\mathcal{Q},{t_0}}{\normsq{\sum_{t={t_0}}^{{t_0}+\period -1}\sum_{n=1}^N q_t^n \sum_{i=0}^{I-1}\g_n(\y^n_{t,i})}}\nonumber \\
    &\overset{(b)}{=} f(\x_{{t_0}}) -\gamma \eta \innerprod{\nabla f(\x_{{t_0}}), \Expectsubbracket{\cdot|\mathcal{Q},{t_0}}{\sum_{t={t_0}}^{{t_0}+\period -1}\sum_{n=1}^N q_t^n \sum_{i=0}^{I-1} \nabla F_n(\y^n_{t,i})}} \nonumber \\
    &\quad \quad + \frac{\gamma^2\eta^2 L}{2}\Expectsubbracket{\cdot|\mathcal{Q},{t_0}}{\normsq{\sum_{t={t_0}}^{{t_0}+\period -1}\sum_{n=1}^N q_t^n \sum_{i=0}^{I-1} \g_n(\y^n_{t,i})}} \nonumber %
\end{align}
where $\Expectsubbracket{\cdot|\mathcal{Q},{t_0}}{\z}$ is a short-hand notation for $\Expectbracket{\z \big|\mathcal{Q}, \x_{{t_0}}}$; $(a)$ is due to the law of total expectation, where the expectation is taken over $\y^n_{t,i}$; $(b)$ is due to the unbiasedness of stochastic gradients. In other parts of the proof, we may use the law of total expectation in a similar way without explanation. 

Taking expectation on both sides over $\x_{t_0}$, we obtain
\begin{align}
    \Expectsubbracket{\cdot|\mathcal{Q}}{f(\x_{{t_0}+\period })} 
    &\leq \Expectsubbracket{\cdot|\mathcal{Q}}{f(\x_{{t_0}})} -\Expectsubbracket{\cdot|\mathcal{Q}}{\gamma \eta \innerprod{\nabla f(\x_{{t_0}}), \sum_{t={t_0}}^{{t_0}+\period -1}\sum_{n=1}^N q_t^n \sum_{i=0}^{I-1}\nabla F_n(\y^n_{t,i})}} \nonumber\\
    &\quad\quad + \frac{\gamma^2\eta^2 L}{2}\Expectsubbracket{\cdot|\mathcal{Q}}{\normsq{\sum_{t={t_0}}^{{t_0}+\period -1}\sum_{n=1}^N q_t^n \sum_{i=0}^{I-1}\g_n(\y^n_{t,i})}}. \label{eq:proofSmoothness}
\end{align}

Consider the second term in \eqref{eq:proofSmoothness},
\begin{align*}
    & -\gamma \eta \innerprod{\nabla f(\x_{{t_0}}), \sum_{t={t_0}}^{{t_0}+\period -1}\sum_{n=1}^N q_t^n \sum_{i=0}^{I-1}\nabla F_n(\y^n_{t,i})}\\
    &= -\frac{\gamma \eta}{I\period } \innerprod{I\period  \nabla f(\x_{{t_0}}), \sum_{t={t_0}}^{{t_0}+\period -1}\sum_{n=1}^N q_t^n \sum_{i=0}^{I-1}\nabla F_n(\y^n_{t,i})}\\
    &\overset{(a)}{=} \frac{\gamma \eta}{2I\period } \normsq{\sum_{t={t_0}}^{{t_0}+\period -1}\sum_{n=1}^N q_t^n \sum_{i=0}^{I-1}\nabla F_n(\y^n_{t,i}) - I\period \nabla f(\x_{{t_0}})} \\
    & \quad\quad  - \frac{\gamma \eta I\period }{2} \normsq{\nabla f(\x_{{t_0}})}  - \frac{\gamma \eta}{2I\period } \normsq{\sum_{t={t_0}}^{{t_0}+\period -1}\sum_{n=1}^N q_t^n \sum_{i=0}^{I-1}\nabla F_n(\y^n_{t,i})}
    \\
    &= \frac{\gamma \eta}{2I\period } \normsq{\sum_{t={t_0}}^{{t_0}+\period -1}\sum_{n=1}^N q_t^n \sum_{i=0}^{I-1} \left[ \nabla F_n(\y^n_{t,i}) - \nabla f(\x_{{t_0}})\right]} \\
    & \quad\quad - \frac{\gamma \eta I\period }{2} \normsq{\nabla f(\x_{{t_0}})} - \frac{\gamma \eta}{2I\period } \normsq{\sum_{t={t_0}}^{{t_0}+\period -1}\sum_{n=1}^N q_t^n \sum_{i=0}^{I-1}\nabla F_n(\y^n_{t,i})}
    \\
    &= \frac{\gamma \eta}{2I\period } \Bigg\Vert\sum_{t={t_0}}^{{t_0}+\period -1}\sum_{n=1}^N q_t^n \sum_{i=0}^{I-1} \Bigg[ \nabla F_n(\y^n_{t,i}) - \nabla F_n\left(\sum_{n'=1}^N q_t^{n'} \y^{n'}_{t,i}\right) + \nabla F_n\left(\sum_{n'=1}^N q_t^{n'} \y^{n'}_{t,i}\right) \\
    & \quad\quad\quad\quad\quad\quad\quad\quad\quad\quad - \nabla F_n(\x_{{t_0}}) + \nabla F_n(\x_{{t_0}}) - \nabla f(\x_{{t_0}})\Bigg]\Bigg\Vert^2 \\
    &\quad\quad
    - \frac{\gamma \eta I\period }{2} \normsq{\nabla f(\x_{{t_0}})}
    - \frac{\gamma \eta}{2I\period } \normsq{\sum_{t={t_0}}^{{t_0}+\period -1}\sum_{n=1}^N q_t^n \sum_{i=0}^{I-1}\nabla F_n(\y^n_{t,i})}
    \\
    &\overset{(b)}{\leq} \frac{3\gamma \eta}{2I\period } \normsq{\sum_{t={t_0}}^{{t_0}+\period -1}\sum_{n=1}^N q_t^n \sum_{i=0}^{I-1} \left[ \nabla F_n(\y^n_{t,i}) - \nabla F_n\left(\sum_{n'=1}^N q_t^{n'} \y^{n'}_{t,i}\right)\right]} \\  
    &\quad\quad + \frac{3\gamma \eta}{2I\period } \normsq{\sum_{t={t_0}}^{{t_0}+\period -1}\sum_{n=1}^N q_t^n \sum_{i=0}^{I-1} \left[ \nabla F_n\left(\sum_{n'=1}^N q_t^{n'} \y^{n'}_{t,i}\right) - \nabla F_n(\x_{{t_0}})\right]}\\
    &\quad\quad + \frac{3\gamma \eta}{2I\period } \normsq{\sum_{t={t_0}}^{{t_0}+\period -1}\sum_{n=1}^N q_t^n \sum_{i=0}^{I-1} \left[ \nabla F_n(\x_{{t_0}}) - \nabla f(\x_{{t_0}})\right]} \\
    &\quad\quad
    - \frac{\gamma \eta I\period }{2} \normsq{\nabla f(\x_{{t_0}})}
    - \frac{\gamma \eta}{2I\period } \normsq{\sum_{t={t_0}}^{{t_0}+\period -1}\sum_{n=1}^N q_t^n \sum_{i=0}^{I-1}\nabla F_n(\y^n_{t,i})}
    \\
    &\overset{(c)}{\leq} \frac{3\gamma \eta}{2} \sum_{t={t_0}}^{{t_0}+\period -1}\sum_{n=1}^N  \sum_{i=0}^{I-1} q_t^n \normsq{  \nabla F_n(\y^n_{t,i}) - \nabla F_n\left(\sum_{n'=1}^N q_t^{n'} \y^{n'}_{t,i}\right)}\\
    &\quad\quad + \frac{3\gamma \eta}{2} \sum_{t={t_0}}^{{t_0}+\period -1}\sum_{n=1}^N  \sum_{i=0}^{I-1} q_t^n \normsq{  \nabla F_n\left(\sum_{n'=1}^N q_t^{n'} \y^{n'}_{t,i}\right) - \nabla F_n(\x_{{t_0}})} \\
    &\quad\quad + \frac{3\gamma \eta I\period  \tilde{\delta}^2(\period)}{2}  
    - \frac{\gamma \eta I\period }{2} \normsq{\nabla f(\x_{{t_0}})}
    - \frac{\gamma \eta}{2I\period } \normsq{\sum_{t={t_0}}^{{t_0}+\period -1}\sum_{n=1}^N q_t^n \sum_{i=0}^{I-1}\nabla F_n(\y^n_{t,i})}
    \\
    &\leq \frac{3\gamma \eta L^2}{2} \sum_{t={t_0}}^{{t_0}+\period -1}\sum_{n=1}^N  \sum_{i=0}^{I-1} q_t^n \normsq{ \y^n_{t,i} - \sum_{n'=1}^N q_t^{n'} \y^{n'}_{t,i}}\\
    &\quad\quad + \frac{3\gamma \eta L^2}{2} \sum_{t={t_0}}^{{t_0}+\period -1} \sum_{i=0}^{I-1} \normsq{  \sum_{n'=1}^N q_t^{n'} \y^{n'}_{t,i} - \x_{{t_0}}} \\
    &\quad\quad + \frac{3\gamma \eta I\period  \tilde{\delta}^2(\period)}{2}  
    - \frac{\gamma \eta I\period }{2} \normsq{\nabla f(\x_{{t_0}})}
    - \frac{\gamma \eta}{2I\period } \normsq{\sum_{t={t_0}}^{{t_0}+\period -1}\sum_{n=1}^N q_t^n \sum_{i=0}^{I-1}\nabla F_n(\y^n_{t,i})}
\end{align*}
where $(a)$ uses $-\innerprod{\mathbf{a}, \mathbf{b}} = \frac{1}{2}(\normsq{\mathbf{a}-\mathbf{b}} - \normsq{\mathbf{a}} - \normsq{\mathbf{b}})$;
$(b)$ uses Jensen's inequality in the first term; $(c)$ uses Jensen's inequality in the first two terms and the definition of $\tilde{\delta}^2(\period)$ in the third term.

Hence,
\begin{align}
    & -\Expectsubbracket{\cdot|\mathcal{Q}}{\gamma \eta \innerprod{\nabla f(\x_{{t_0}}), \sum_{t={t_0}}^{{t_0}+\period -1}\sum_{n=1}^N q_t^n \sum_{i=0}^{I-1}\nabla F_n(\y^n_{t,i})}} \nonumber\\
    &\leq \frac{3\gamma \eta L^2}{2} \sum_{t={t_0}}^{{t_0}+\period -1}\sum_{n=1}^N  \sum_{i=0}^{I-1} q_t^n \Expectsubbracket{\cdot|\mathcal{Q}}{\normsq{ \y^n_{t,i} - \sum_{n'=1}^N q_t^{n'} \y^{n'}_{t,i}}} \nonumber\\
    &\quad + \frac{3\gamma \eta L^2}{2} \sum_{t={t_0}}^{{t_0}+\period -1} \sum_{i=0}^{I-1} \Expectsubbracket{\cdot|\mathcal{Q}}{\normsq{  \sum_{n'=1}^N q_t^{n'} \y^{n'}_{t,i} - \x_{{t_0}}}} \nonumber \\
    &\quad+ \frac{3\gamma \eta I\period  \tilde{\delta}^2(\period)}{2}  
    - \frac{\gamma \eta I\period }{2} \Expectsubbracket{\cdot|\mathcal{Q}}{\normsq{\nabla f(\x_{{t_0}})}}
    - \frac{\gamma \eta}{2I\period } \Expectsubbracket{\cdot|\mathcal{Q}}{\normsq{\sum_{t={t_0}}^{{t_0}+\period -1}\sum_{n=1}^N q_t^n \sum_{i=0}^{I-1}\nabla F_n(\y^n_{t,i})}} \nonumber\\
    &\overset{(a)}{\leq} \frac{3\gamma \eta L^2}{2} \sum_{t={t_0}}^{{t_0}+\period -1} \sum_{i=0}^{I-1} 3I\gamma^2\left( 6I \tilde{\nu}^2 +  \sigma^2 \left(1-\sum_{n=1}^N (q_t^n)^2\right)\right) \nonumber\\
    &\quad + \frac{3\gamma \eta L^2}{2} \sum_{t={t_0}}^{{t_0}+\period -1} \sum_{i=0}^{I-1} \Bigg(3\gamma^2 I^2 \tilde{\nu}^2  +   24\gamma^2 I^2 \period ^2  \tilde{\beta}^2  + 24 \gamma^2 I^2 \period ^2 \Expectsubbracket{\cdot|\mathcal{Q}}{\normsq{\nabla f(\x_{t_0})}} \nonumber\\
    &\quad\quad\quad\quad\quad\quad\quad\quad\quad\quad\quad\quad + 3 \gamma^2 I\period  \sigma^2 \left[ \rho^2   \!+\! \frac{1}{6\period }  \left(1\!-\!\sum_{n=1}^N (q_t^n)^2\right) \right] \Bigg)  \nonumber\\
    &\quad+ \frac{3\gamma \eta I\period  \tilde{\delta}^2(\period)}{2}  
    - \frac{\gamma \eta I\period }{2} \Expectsubbracket{\cdot|\mathcal{Q}}{\normsq{\nabla f(\x_{{t_0}})}}
    - \frac{\gamma \eta}{2I\period } \Expectsubbracket{\cdot|\mathcal{Q}}{\normsq{\sum_{t={t_0}}^{{t_0}+\period -1}\sum_{n=1}^N q_t^n \sum_{i=0}^{I-1}\nabla F_n(\y^n_{t,i})}} \nonumber\\
    &= \frac{9\gamma^3 \eta L^2 I^2\period }{2} \left( 6I \tilde{\nu}^2 +  \sigma^2 \left(1-\sum_{n=1}^N (q_t^n)^2\right)\right) \nonumber\\
    &\quad + \frac{3\gamma \eta L^2 I\period }{2}  \Bigg(3\gamma^2 I^2 \tilde{\nu}^2  +   24\gamma^2 I^2 \period ^2  \tilde{\beta}^2  + 24 \gamma^2 I^2 \period ^2 \Expectsubbracket{\cdot|\mathcal{Q}}{\normsq{\nabla f(\x_{t_0})}} \nonumber\\
    &\quad\quad\quad\quad\quad\quad\quad\quad\quad\quad\quad\quad + 3 \gamma^2 I\period  \sigma^2 \left[ \rho^2   + \frac{1}{6\period }  \left(1-\sum_{n=1}^N (q_t^n)^2\right) \right]\Bigg)  \nonumber\\
    &\quad+ \frac{3\gamma \eta I\period  \tilde{\delta}^2(\period)}{2}  
    - \frac{\gamma \eta I\period }{2} \Expectsubbracket{\cdot|\mathcal{Q}}{\normsq{\nabla f(\x_{{t_0}})}}
    - \frac{\gamma \eta}{2I\period } \Expectsubbracket{\cdot|\mathcal{Q}}{\normsq{\sum_{t={t_0}}^{{t_0}+\period -1}\sum_{n=1}^N q_t^n \sum_{i=0}^{I-1}\nabla F_n(\y^n_{t,i})}}
    \label{eq:proofTheoremIntermediate1}
\end{align}
where $(a)$ uses Lemmas~\ref{lemma:yDivergence} and \ref{lemma:overallChange}.

Note also that
\begin{align}
    &\Expectsubbracket{\cdot|\mathcal{Q}}{\normsq{\sum_{t={t_0}}^{{t_0}+\period -1}\sum_{n=1}^N q_t^n \sum_{i=0}^{I-1}\g_n(\y^n_{t,i})}} \nonumber\\
    &\overset{(a)}{\leq} 2\Expectsubbracket{\cdot|\mathcal{Q}}{\normsq{\sum_{t={t_0}}^{{t_0}+\period -1}\sum_{n=1}^N q_t^n \sum_{i=0}^{I-1}\nabla F_n(\y^n_{t,i})}} \nonumber\\
    &\quad\quad + 2\Expectsubbracket{\cdot|\mathcal{Q}}{\normsq{\sum_{t={t_0}}^{{t_0}+\period -1}\sum_{n=1}^N q_t^n \sum_{i=0}^{I-1}\left[\g_n(\y^n_{t,i}) - \nabla F_n(\y^n_{t,i})\right]}} \nonumber\\
    &\overset{(b)}{=} 2\Expectsubbracket{\cdot|\mathcal{Q}}{\normsq{\sum_{t={t_0}}^{{t_0}+\period -1}\sum_{n=1}^N q_t^n \sum_{i=0}^{I-1}\nabla F_n(\y^n_{t,i})}} \nonumber\\
    &\quad\quad + 2\sum_{t={t_0}}^{{t_0}+\period -1}\sum_{n=1}^N \left(q_t^n\right)^2 \sum_{i=0}^{I-1} \Expectsubbracket{\cdot|\mathcal{Q}}{\normsq{\g_n(\y^n_{t,i}) - \nabla F_n(\y^n_{t,i})}} \nonumber\\
    &\overset{(c)}{\leq} 2\Expectsubbracket{\cdot|\mathcal{Q}}{\normsq{\sum_{t={t_0}}^{{t_0}+\period -1}\sum_{n=1}^N q_t^n \sum_{i=0}^{I-1}\nabla F_n(\y^n_{t,i})}} + 2I \period  \rho^2 \sigma^2 
    \label{eq:proofTheoremIntermediate2}
\end{align}
where $(a)$ uses Jensen's inequality; $(b)$ is because the cross terms after expanding the norm in the second term has zero mean, due to the unbiasedness of the stochastic gradient for any parameter $\y^n_{t,i}$ which implies that $\g_n(\y^n_{t,i}) - \nabla F_n(\y^n_{t,i})$ has zero mean when conditioned on any $\y^n_{t,i}$;
% the stochastic gradient noise is independent across clients and across time, and the variance of the sum of independent random variables is equal to the sum of the variance and $\mathrm{Var}(a\z) = a^2 \mathrm{Var}(\z)$ where $\mathrm{Var}(\z) := \Expectsubbracket{\cdot|\mathcal{Q}}{\normsq{\z - \Expectsubbracket{\cdot|\mathcal{Q}}{\z}}}$;
$(c)$ uses the definitions of $\sigma^2$ and $\rho^2$.

Plugging \eqref{eq:proofTheoremIntermediate1} and \eqref{eq:proofTheoremIntermediate2} back into \eqref{eq:proofSmoothness}, we get
\begin{align*}
    &\Expectsubbracket{\cdot|\mathcal{Q}}{f(\x_{{t_0}+\period })}\nonumber
    \\ 
    &\leq \Expectsubbracket{\cdot|\mathcal{Q}}{f(\x_{{t_0}})} -\Expectsubbracket{\cdot|\mathcal{Q}}{\gamma \eta \innerprod{\nabla f(\x_{{t_0}}), \sum_{t={t_0}}^{{t_0}+\period -1}\sum_{n=1}^N q_t^n \sum_{i=0}^{I-1}\nabla F_n(\y^n_{t,i})}} \\
    &\quad + \frac{\gamma^2\eta^2 L}{2}\Expectsubbracket{\cdot|\mathcal{Q}}{\normsq{\sum_{t={t_0}}^{{t_0}+\period -1}\sum_{n=1}^N q_t^n \sum_{i=0}^{I-1}\g_n(\y^n_{t,i})}} \\
    &\leq \Expectsubbracket{\cdot|\mathcal{Q}}{f(\x_{{t_0}})} + \frac{9 \gamma^3 \eta L^2 I^2\period }{2} \left( 6I \tilde{\nu}^2 +  \sigma^2 \left(1-\sum_{n=1}^N (q_t^n)^2\right)\right)\\
    &\quad + \frac{3\gamma \eta L^2 I\period }{2}  \left(3\gamma^2 I^2 \tilde{\nu}^2  +   24\gamma^2 I^2 \period ^2  \tilde{\beta}^2  + 24 \gamma^2 I^2 \period ^2 \Expectsubbracket{\cdot|\mathcal{Q}}{\normsq{\nabla f(\x_{t_0})}} \right.\\
    &\quad\quad\quad\quad\quad\quad\quad\quad\quad\quad \left.+ 3 \gamma^2 I\period  \sigma^2 \left[ \rho^2   + \frac{1}{6\period }  \left(1-\sum_{n=1}^N (q_t^n)^2\right) \right] \right) \\
    &\quad+ \frac{3\gamma \eta I\period  \tilde{\delta}^2(\period)}{2}  
    - \frac{\gamma \eta I\period }{2} \Expectsubbracket{\cdot|\mathcal{Q}}{\normsq{\nabla f(\x_{{t_0}})}}
    - \frac{\gamma \eta}{2I\period } \Expectsubbracket{\cdot|\mathcal{Q}}{\normsq{\sum_{t={t_0}}^{{t_0}+\period -1}\sum_{n=1}^N q_t^n \sum_{i=0}^{I-1}\nabla F_n(\y^n_{t,i})}}\\
    &\quad + \gamma^2\eta^2 L\cdot\Expectsubbracket{\cdot|\mathcal{Q}}{\normsq{\sum_{t={t_0}}^{{t_0}+\period -1}\sum_{n=1}^N q_t^n \sum_{i=0}^{I-1}\nabla F_n(\y^n_{t,i})}} + \gamma^2\eta^2 L  I \period  \rho^2 \sigma^2 \\
    &= \Expectsubbracket{\cdot|\mathcal{Q}}{f(\x_{{t_0}})} + \frac{9 \gamma^3 \eta L^2 I^2\period }{2}  \left( 6I \tilde{\nu}^2 +  \sigma^2 \left(1-\sum_{n=1}^N (q_t^n)^2\right)\right)\\
    &\quad + \frac{3\gamma \eta L^2 I\period }{2}  \left(3\gamma^2 I^2 \tilde{\nu}^2  +   24\gamma^2 I^2 \period ^2  \tilde{\beta}^2   + 3 \gamma^2 I\period  \sigma^2 \left[ \rho^2   + \frac{1}{6\period }  \left(1-\sum_{n=1}^N (q_t^n)^2\right) \right] \right) \\
    &\quad+ \frac{3\gamma \eta I\period  \tilde{\delta}^2(\period)}{2}  
    - \frac{\gamma \eta I\period }{2} \left(1 - 72 \gamma^2 L^2 I^2 \period ^2\right) \Expectsubbracket{\cdot|\mathcal{Q}}{\normsq{\nabla f(\x_{{t_0}})}}\\  
    &\quad
    - \left(\frac{\gamma \eta}{2I\period } - \gamma^2\eta^2 L\right) \Expectsubbracket{\cdot|\mathcal{Q}}{\normsq{\sum_{t={t_0}}^{{t_0}+\period -1}\sum_{n=1}^N q_t^n \sum_{i=0}^{I-1}\nabla F_n(\y^n_{t,i})}}\\
    &\quad + \gamma^2\eta^2 L  I \period  \rho^2 \sigma^2 \\
    &\leq \Expectsubbracket{\cdot|\mathcal{Q}}{f(\x_{{t_0}})} + \frac{9 \gamma^3 \eta L^2 I^2\period }{2} \left( 6I \tilde{\nu}^2 +  \sigma^2 \left(1-\sum_{n=1}^N (q_t^n)^2\right)\right)\\
    &\quad + \frac{3\gamma \eta L^2 I\period }{2}  \left(3\gamma^2 I^2 \tilde{\nu}^2  +   24\gamma^2 I^2 \period ^2  \tilde{\beta}^2   + 3 \gamma^2 I\period  \sigma^2 \left[ \rho^2   + \frac{1}{6\period }  \left(1-\sum_{n=1}^N (q_t^n)^2\right) \right]  \right) \\
    &\quad+ \frac{3\gamma \eta I\period  \tilde{\delta}^2(\period)}{2}- \frac{ \gamma \eta I\period }{4}  \Expectsubbracket{\cdot|\mathcal{Q}}{\normsq{\nabla f(\x_{{t_0}})}}\\
    &\quad + \gamma^2\eta^2 L  I \period  \rho^2 \sigma^2 
\end{align*}
where the last inequality is because $\gamma \leq \frac{1}{12 LI\period }$ thus $-(1 - 72 \gamma^2 L^2 I^2 \period ^2) \leq -\frac{1}{2} $, and $\gamma \eta \leq \frac{1}{2LI\period }$ thus $\frac{1}{2I\period } - \gamma \eta L \geq 0$.

Rearranging, we have
\begin{align*}
    &\frac{ \gamma \eta I\period }{4}  \Expectsubbracket{\cdot|\mathcal{Q}}{\normsq{\nabla f(\x_{{t_0}})}}\\
    &\leq \Expectsubbracket{\cdot|\mathcal{Q}}{f(\x_{{t_0}})} - \Expectsubbracket{\cdot|\mathcal{Q}}{f(\x_{{t_0}+\period })} + \frac{9 \gamma^3 \eta L^2 I^2\period }{2} \left( 6I \tilde{\nu}^2 +  \sigma^2 \left(1-\sum_{n=1}^N (q_t^n)^2\right)\right)\\
    &\quad + \frac{3\gamma \eta L^2 I\period }{2}  \left(3\gamma^2 I^2 \tilde{\nu}^2  +   24\gamma^2 I^2 \period ^2  \tilde{\beta}^2   + 3 \gamma^2 I\period  \sigma^2 \left[ \rho^2   + \frac{1}{6\period }  \left(1-\sum_{n=1}^N (q_t^n)^2\right) \right]  \right) \\
    &\quad + \frac{3\gamma \eta I\period  \tilde{\delta}^2(\period)}{2}  + \gamma^2\eta^2 L  I \period  \rho^2 \sigma^2 .
\end{align*}

Hence,
\begin{align}
    &\Expectsubbracket{\cdot|\mathcal{Q}}{\normsq{\nabla f(\x_{{t_0}})}} \nonumber \\
    &\leq \frac{4\left(\Expectsubbracket{\cdot|\mathcal{Q}}{f(\x_{{t_0}})} - \Expectsubbracket{\cdot|\mathcal{Q}}{f(\x_{{t_0}+\period })}\right)}{\gamma \eta I\period } + 18\gamma^2 L^2 I \left( 6I \tilde{\nu}^2 +  \sigma^2 \left(1-\sum_{n=1}^N (q_t^n)^2\right)\right)\nonumber\\
    &\quad + 6 L^2  \left(3\gamma^2 I^2 \tilde{\nu}^2  +   24\gamma^2 I^2 \period ^2  \tilde{\beta}^2   + 3 \gamma^2 I\period  \sigma^2 \left[ \rho^2   + \frac{1}{6\period }  \left(1-\sum_{n=1}^N (q_t^n)^2\right) \right]  \right) \nonumber\\
    &\quad + 6 \tilde{\delta}^2(\period)  + 4\gamma\eta L \rho^2 \sigma^2 \nonumber\\
    &=\frac{4\left(\Expectsubbracket{\cdot|\mathcal{Q}}{f(\x_{{t_0}})} - \Expectsubbracket{\cdot|\mathcal{Q}}{f(\x_{{t_0}+\period })}\right)}{ \gamma \eta I\period } + 126 \gamma^2 L^2 I^2 \tilde{\nu}^2 + 144 \gamma^2 L^2 I^2 \period ^2  \tilde{\beta}^2 + 6 \tilde{\delta}^2(\period) \nonumber\\
    &\quad + 18\gamma^2 L^2 I \sigma^2 \left(1-\sum_{n=1}^N (q_t^n)^2\right) + 18 \gamma^2 L^2 I\period  \sigma^2 \left[ \rho^2   + \frac{1}{6\period }  \left(1-\sum_{n=1}^N (q_t^n)^2\right) \right] + 4\gamma\eta L \rho^2 \sigma^2 \nonumber\\
    &=\frac{4\left(\Expectsubbracket{\cdot|\mathcal{Q}}{f(\x_{{t_0}})} - \Expectsubbracket{\cdot|\mathcal{Q}}{f(\x_{{t_0}+\period })}\right)}{ \gamma \eta I\period } + 126 \gamma^2 L^2 I^2 \tilde{\nu}^2 + 144 \gamma^2 L^2 I^2 \period ^2  \tilde{\beta}^2 + 6 \tilde{\delta}^2(\period) \nonumber\\
    &\quad + 21\gamma^2 L^2 I \sigma^2 \left(1-\sum_{n=1}^N (q_t^n)^2\right) + 18 \gamma^2 L^2 I\period  \sigma^2 \rho^2  + 4\gamma\eta L \rho^2 \sigma^2 \nonumber\\
    &\leq \frac{4\left(\Expectsubbracket{\cdot|\mathcal{Q}}{f(\x_{{t_0}})} - \Expectsubbracket{\cdot|\mathcal{Q}}{f(\x_{{t_0}+\period })}\right)}{ \gamma \eta I\period } + 126 \gamma^2 L^2 I^2 \tilde{\nu}^2 + 144 \gamma^2 L^2 I^2 \period ^2  \tilde{\beta}^2 + 6 \tilde{\delta}^2(\period) \nonumber\\
    &\quad + \left(21\gamma^2 L^2 I  + 18 \gamma^2 L^2 I\period   \rho^2  + 4\gamma\eta L \rho^2\right) \sigma^2 . \label{eq:proofTheoremBeforeTelescoping}
\end{align}

We now have
\begin{align*}
    &\min_{t} \Expectsubbracket{\cdot|\mathcal{Q}}{\normsq{\nabla f(\x_{{t}})}}\\
    &\leq \min_{t_0 \in \left\{0, \period , 2\period , ..., \left(\left\lfloor \frac{T}{\period }\right\rfloor -1\right) \period \right\}} \Expectsubbracket{\cdot|\mathcal{Q}}{\normsq{\nabla f(\x_{{t_0}})}} \\
    &\leq \frac{1}{\left\lfloor \sfrac{T}{\period }\right\rfloor}\cdot \sum_{{t_0}=0, \period , 2\period , ..., \left(\left\lfloor \frac{T}{\period }\right\rfloor -1\right) \period } \Expectsubbracket{\cdot|\mathcal{Q}}{\normsq{\nabla f(\x_{{t_0}})}} \\
    &\overset{(a)}{\leq} \frac{8\left(f(\x_0) - f^*\right)}{\gamma \eta IT} + 126 \gamma^2 L^2 I^2 \tilde{\nu}^2 + 144 \gamma^2 L^2 I^2 \period ^2  \tilde{\beta}^2 + 6 \tilde{\delta}^2(\period) \\
    &\quad\quad + \left(21\gamma^2 L^2 I  + 18 \gamma^2 L^2 I\period   \rho^2  + 4\gamma\eta L \rho^2\right) \sigma^2 \\
    &= \mathcal{O}\left( \frac{\mathcal{F}}{\gamma \eta IT} +   \gamma^2 L^2 I^2 \tilde{\nu}^2 +   \gamma^2 L^2 I^2 \period ^2  \tilde{\beta}^2 +   \tilde{\delta}^2(\period) + \left( \gamma^2 L^2 I  +   \gamma^2 L^2 I\period   \rho^2  +  \gamma\eta L \rho^2\right) \sigma^2  \right)
\end{align*}
where the first term in $(a)$ is from telescoping of \eqref{eq:proofTheoremBeforeTelescoping} and $\frac{1}{\left\lfloor \sfrac{T}{\period }\right\rfloor \period } \leq \frac{1}{\left(\sfrac{T}{\period } -1\right)\period } = \frac{1}{T-\period } \leq \frac{1}{T-\sfrac{T}{2}} = \frac{2}{T}$ since $\period \leq \frac{T}{2}$, and the other terms in $(a)$ are constants independent of $t_0$.
\qed

\subsection{Proof of Corollaries~\ref{corollary:mainConvergenceResult1} and \ref{corollary:mainConvergenceResult2}}

For Corollary~\ref{corollary:mainConvergenceResult1}, we note that $\gamma\eta = \min\left\{\frac{\sqrt{\mathcal{F}}}{\sigma\rho\sqrt{LIT}}; \frac{1}{2LI\period }\right\}$ by the choice of $\gamma$ and $\eta$. Because $\gamma \leq \frac{1}{12LI\period }$ (since $T\geq 1$) and $\gamma\eta \leq \frac{1}{2LI\period }$, the conditions of Theorem~\ref{theorem:mainConvergenceResult} hold. Then, using $\gamma\eta \leq \frac{\sqrt{\mathcal{F}}}{\sigma\rho\sqrt{LIT}}$, $\frac{1}{\gamma\eta} = \max\left\{\frac{\sigma\rho\sqrt{LIT}}{\sqrt{\mathcal{F}}}; 2LI\period \right\} \leq \frac{\sigma\rho\sqrt{LIT}}{\sqrt{\mathcal{F}}} + 2LI\period $, and $\rho^2\leq 1$, we obtain the result.

For Corollary~\ref{corollary:mainConvergenceResult2}, we have $\gamma \eta = \frac{\sqrt{\mathcal{F}}}{\rho\sqrt{LIT}}$ according to the choice of $\gamma$ and $\eta$. Because $\frac{\sqrt{\mathcal{F}}}{\rho\sqrt{LIT}} \leq \frac{1}{2LIP}$, we have $\gamma \eta \leq \frac{1}{2LIP}$ and the conditions of Theorem~\ref{theorem:mainConvergenceResult} hold. The result is obtained by plugging the values of $\gamma$ and $\eta$ into the result in Theorem~\ref{theorem:mainConvergenceResult}, and noting that $\rho^2\leq 1$ in the last term.
\qed

\section{Additional Details and Results of Experiments}
\label{sec:appendixExperiment}

\subsection{Additional Details of the Setup}
\label{subsec:appendixExperimentSetup}

\textbf{Code.}
Please visit \url{https://shiqiang.wang/code/fl-arbitrary-participation}

\textbf{Datasets.}
The FashionMNIST dataset\footnote{\url{https://github.com/zalandoresearch/fashion-mnist}} has MIT license and a citation requirement \cite{FashionMNIST}. The CIFAR-10 dataset\footnote{\url{https://www.cs.toronto.edu/~kriz/cifar.html}} only has a citation requirement \cite{CIFAR10}. We have cited both papers in our main paper.
Our experiments used the original split of training and test data in each dataset, where the training data is further split into workers in a non-IID manner according to the labels of data samples with a similarity of $5\%$, as described in the main paper.

\textbf{Motivation of the Setup.}
The non-IID splitting of data to clients and the periodic connectivity are motivated by application scenarios of FL. An example of a real-world scenario is where different geographically-dispersed user groups can participate in training at different times of the day, e.g., when their phones are charging during the night. Another example is where multiple organizations have their local servers (also known as edge servers), which can participate in training only when their servers are idle, such as during the night when there are few user requests, and these servers are located in different geographical regions. Our notion of ``client'' is a general term that can stand for the phone or the local server in these specific scenarios.

\textbf{Availability of Clients.}
In the case of \textit{periodically available} clients, as described in the main paper, subsets of clients with different majority class labels take turns to become available. For example, only those clients with the first two majority classes of data are available in the first $100$ rounds; then, in the next $100$ rounds, only those clients with the next two majority classes of data are available, and so on. The majority class of each client is determined by the non-IID data partitioning. There are $10$ classes in total, so a full cycle of participation of all clients is $500$ rounds. To avoid inherent synchronicity across different simulation instances, we apply a random offset in the participation cycle at the beginning of the FL process. With this random offset, the number of rounds where the first subset of clients participate can be decreased to less than $100$.
For example, the first subset of clients may participate for $37$ rounds (an arbitrary number that is less than or equal to $100$), the second subset of clients then participates for the next $100$ rounds, and so on. This random offset varies across different simulation instances. It represents the practical scenario where the FL process can start at any time in the day. In each round, $S=10$ clients are selected to participate, according to a random permutation among the currently available clients. The selection of a small number of clients among the available clients can be due to resource efficiency considerations, for instance.

In addition to periodic availability, we also consider a setting where all the $N=250$ clients are \textit{always available}, and a subset of $S=10$ clients are selected to participate in each round. Here, we consider two participation patterns. For independent participation, the subset of clients are selected randomly according to a uniform distribution, independently across rounds. For regularized participation, the clients are first randomly permuted and then $S$ clients are selected sequentially from the permuted sequence of clients, so that all the $N$ clients participate in FL once in every $\frac{N}{S}$ rounds.

\textbf{Hyperparameter Choices.}
The initial local learning rates of $\gamma=0.1$ and $\gamma=0.05$ for FashionMNIST and CIFAR-10, respectively, were obtained from a learning rate search where the model was trained using $\gamma\in\{0.5, 0.1, 0.05, 0.01, 0.005, 0.001, 0.0005, 0.0001\}$ while fixing $\eta=1$. We considered the case of always available clients for this learning rate search. We chose the learning rate that gave the smallest training loss after $2,000$ rounds of training for FashionMNIST and $4,000$ rounds of training for CIFAR-10, where we observed that the best learning rate is the same for both independent and regularized participation with each dataset. The reason of considering always available clients for tuning the initial learning rates is because this is the ideal case which allows the model to make big steps in training, so it is reasonable to use these learning rates as starting points even in the case of periodic availability. The number of rounds that use the initial learning rate in the case of periodic availability was determined so that the model's performance no longer improves substantially afterwards, if keeping the same initial learning rate.

The subsequent learning rates for the methods shown in Figure \ref{fig:periodicAvailable}, with periodically available clients, are found from a search on a grid that is $\{1, 0.1, 0.01, 0.001, 0.0001\}$ times the initial learning rate, as described in the main paper. The optimal learning rates found from this search are shown in Table~\ref{table:optimalLearningRates}.

\begin{table}[ht]
\caption{Optimal learning rates found from grid search} 
\centering
\label{table:optimalLearningRates}
{
\small
\renewcommand{\arraystretch}{1.5}
\begin{tabular}{>{\centering\arraybackslash}p{0.4\linewidth} | >{\centering\arraybackslash}p{0.15\linewidth} | >{\centering\arraybackslash}p{0.15\linewidth} } 
\thickhline
Method & FashionMNIST & CIFAR-10\\
\thickhline
wait-minibatch & $\gamma = 0.1$ & $\gamma = 0.05$ \\
wait-full & $\gamma = 0.1$ & $\gamma = 0.05$ \\
Algorithm~\ref{alg:main-alg} without amplification ($\eta = 1$) & $\gamma = 1 \times 10^{-5}$ & $\gamma = 5 \times 10^{-5}$ \\
Algorithm~\ref{alg:main-alg} with amplification ($\eta = 10, P=500$) & $\gamma = 1 \times 10^{-5}$ & $\gamma = 5 \times 10^{-6}$ \\
\thickhline
\end{tabular}
}
\end{table}

We fix the number of local updates to $I=5$ and use a minibatch size of $16$ for local SGD at each participating client. These parameters were not tuned specifically, since the relative difference between different methods will likely follow the same trend for different $I$ and minibatch sizes. Because $I$ is a controllable parameter in our experiments, we consider a minor extension (compared to our discussion in Section~\ref{sec:discussions}) in the ``wait-*'' methods in our experiments, where each client performs $I$ local iterations (instead of one iteration in the discussion in Section~\ref{sec:discussions}) before starting to wait. The intuition remains the same as what has been discussed in Section~\ref{sec:discussions}. The reason for this extension is because the ``wait-*'' methods defined in this way is equivalent to the case where all the clients are available but the total number of rounds is reduced by $P$, so it has an easily interpretable meaning in practice. In addition, changing $I$ can change the performance of all methods (including Algorithm~\ref{alg:main-alg}, both with and without amplification, and the ``wait-*'' methods), while their relative difference can remain similar. We do not study the effect of $I$ and minibatch size in details in this work.

\textbf{Model Architecture.}
The CNN architectures used in our experiments are shown in Table~\ref{table:CNNArchitecture}, where we use Kaiming initialization \cite{he2015delving} for all the ReLU layers. All convolutional and dense layers, except for the last layer, use ReLU activation.

\begin{table}[ht]
\caption{CNN architectures} 
\centering
\label{table:CNNArchitecture}
{
\small
\renewcommand{\arraystretch}{1.5}
\begin{tabular}{>{\centering\arraybackslash}p{0.4\linewidth} | >{\centering\arraybackslash}p{0.4\linewidth} } 
\thickhline
FashionMNIST & CIFAR-10\\
\thickhline
Convolutional (input: $1$, kernel: $5\times 5$, padding: 2, output: $32$) & Convolutional (input: $3$, kernel: $5\times 5$, padding: 2, output: $32$) 
\\
MaxPool (kernel: $2\times 2$, stride: $2$) & MaxPool (kernel: $2\times 2$, stride: $2$) 
\\
Convolutional (input: $32$, kernel: $5\times 5$, padding: 2, output: $32$) & Convolutional (input: $32$, kernel: $5\times 5$, padding: 2, output: $64$)
\\
MaxPool (kernel: $2\times 2$, stride: $2$) & MaxPool (kernel: $2\times 2$, stride: $2$)
\\
Dense (input: $1,568$, output: $128$) & Dense (input: $4,096$, output: $512$)
\\
Dense (input: $128$, output: $10$) & Dense (input: $512$, output: $128$)
\\
Softmax & Dense (input: $128$, output: $10$)
\\
& Softmax 
\\
\thickhline
\end{tabular}
}
\end{table}

\textbf{Collecting Results.}
The plots in Figures \ref{fig:periodicAvailable}--\ref{fig:periodicAvailableDiffP} and all the figures showing additional results in the next section were obtained from simulations with $10$ different random seeds for FashionMNIST and $5$ different random seeds for CIFAR-10. For the best viewing experience, we further applied moving average over a window length of $3\%$ of the data points (on the $x$-axis) in the loss and accuracy results. In all the figures, the solid lines and the shaded areas show the average and standard deviation values, respectively, over the moving average window and all simulation runs for each configuration.

\textbf{Computational Resources.}
The workload of experiments was split between a desktop machine with RTX 3070 GPU and an internal GPU cluster. On RTX 3070, the FashionMNIST experiment with $300,000$ training rounds takes about a day to complete for a single simulation instance; the CIFAR-10 experiment with $600,000$ training rounds takes about three days to complete for a single simulation instance. 
The RTX 3070 has memory that can support simulation with $5$ different random seeds (i.e., $5$ different instances) at the same time. On the internal GPU cluster, the running times are slightly shorter since it has better-performing GPUs.
We also note that the ``wait-*'' methods run faster than Algorithm~\ref{alg:main-alg} in simulation because a lot of rounds are skipped, but this does not mean that they can train models faster in practice, because we consider the number of training rounds as an indication of physical time elapse. 
For example, in the case of periodic availability, each subset of clients is available for $100$ rounds in a cycle. This means that $100$ rounds of training can take place during the physical time duration that these clients are available. Algorithm~\ref{alg:main-alg} chooses to actually train models during this time. The ``wait-*'' methods choose not to train in some of the rounds, but they still need to wait until the time has passed and the next subset of clients becomes available, where the time here is measured by the number of training rounds that includes those rounds that may be skipped by the ``wait-*'' methods.

\subsection{Additional Results}

In the following, we present some additional results obtained from experiments.

\subsubsection{Always Available Clients}
\label{subsec:appendixExperimentResultsAlways}

We first consider the case of always available clients, where a subset of $S=10$ clients are selected in an independent or regularized manner, as discussed in Section~\ref{subsec:appendixExperimentSetup}. We compare the results of independent and regularized participation in Figure~\ref{fig:alwaysAvailable}. The observation is that both methods give similar performance, with regularized participation having a marginal performance improvement. As our theory predicts (see Table~\ref{table:mainResults}), both independent and regularized participation can guarantee convergence, where regularized participation gives a slightly better convergence rate.

\begin{figure}[htb]
\includegraphics[width=\linewidth]{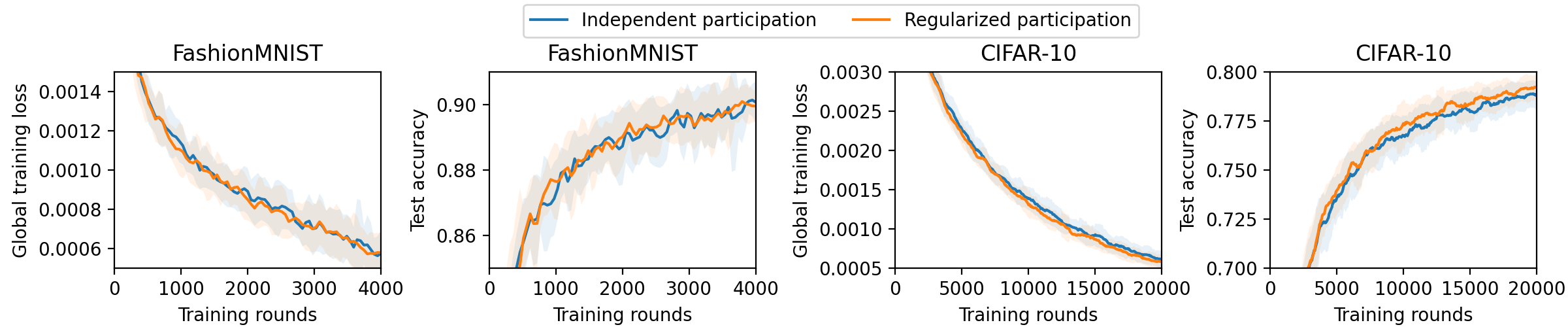}%
\caption{Comparison between independent and regularized participation in the case of always available clients.}
\label{fig:alwaysAvailable}
\end{figure}

\subsubsection{Periodically Available Clients}
\label{subsec:appendixExperimentResultsPeriodic}

\textbf{Different Learning Rates and Amplification Factors.}
Next, we consider periodically available clients as in the main paper. We study the performance of Algorithm~\ref{alg:main-alg} with different learning rate and amplification factor configurations. The results are shown in Figures~\ref{fig:periodicAvailableDiffLR} and \ref{fig:periodicAvailableDiffAmplification}. 

Figure~\ref{fig:periodicAvailableDiffLR} shows the results with different combinations of the local learning rate $\gamma$ and amplification factor $\eta$. We note that $\eta=1$ corresponds to the case without amplification and $\eta=10$ corresponds to the case with amplification (and with an amplification factor of $10$). We can observe that when $\gamma$ is large, the training is generally unstable with high fluctuation, high loss, and low accuracy. Reducing $\gamma$ improves the performance.
According to our theory, $\period=500$ gives a small $\tilde{\delta}^2(\period)$ with this experimental setup, because the cycle of all clients being available once is $500$ rounds. On the other hand, a small $\tilde{\delta}^2(\period)$ can be achieved with a much smaller $\period$ in the case of always available clients (in Section~\ref{subsec:appendixExperimentResultsAlways}). When $\period$ gets large, our theory suggests that the learning rate $\gamma$ should become small in order to guarantee convergence. This explains why non-convergence is observed in Figure~\ref{fig:periodicAvailableDiffLR} for large $\gamma$, while convergence is observed for small $\gamma$. 
For the smallest $\gamma$ in each case, i.e., $\gamma=10^{-5}$ for FashionMNIST and $\gamma=5\times 10^{-6}$ for CIFAR-10, we can see that amplification with $\eta=10$ further improves the performance over no amplification ($\eta=1$).

\begin{figure}[H]
\includegraphics[width=\linewidth]{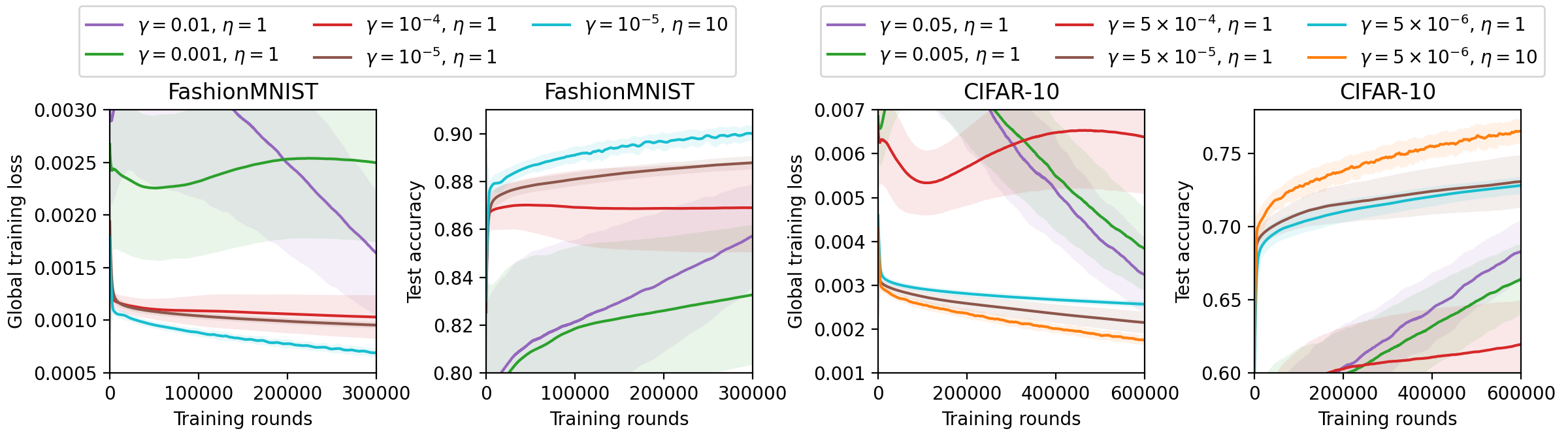}%
\caption{Algorithm~\ref{alg:main-alg} with different learning rates in the case of periodically available clients. Note: FashionMNIST with $\gamma=0.1$ is not shown in the plots due to NaN values observed when the number of training rounds is large.}
\label{fig:periodicAvailableDiffLR}
\vspace{1.5em}
\end{figure}

In Figure~\ref{fig:periodicAvailableDiffAmplification}, we further investigate the impact of different amplification factors $\eta$, while fixing $\gamma=10^{-5}$ for FashionMNIST and $\gamma=5\times 10^{-6}$ for CIFAR-10 as in the case of $\eta=10$ in Figure~\ref{fig:periodicAvailableDiffLR}. We observe that the cases with $\eta=5$, $\eta=10$, and $\eta=15$ have similar performance, especially when getting closer to convergence. The performance gets worse when $\eta$ is too small (e.g., $\eta=2$), because the advantage of amplification is not fully attained in this case. When $\eta$ is too large (e.g., $\eta=20$), we observe higher fluctuation in the loss and accuracy values, which also leads to lower performance. For both datasets and their corresponding models considered in our experiments, we observe that the cases of $\eta=5$, $\eta=10$, and $\eta=15$ all give reasonably good performance. This suggests that a coarse choice of $\eta$ based on experience can be sufficient in practice, without the need of fine-grained tuning.

\begin{figure}[H]
\includegraphics[width=\linewidth]{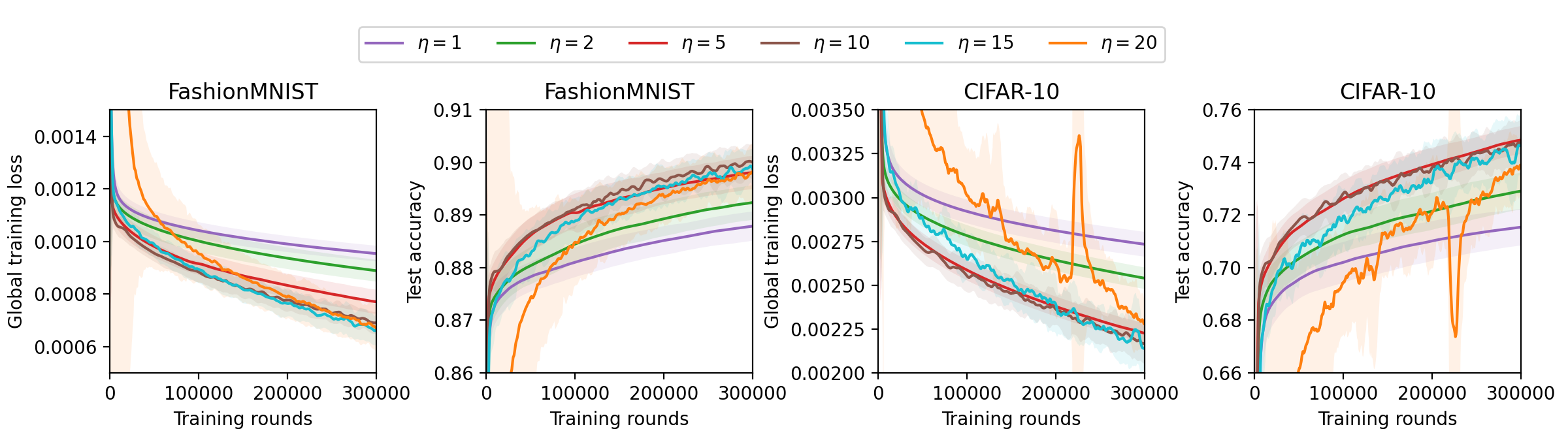}%
\caption{Algorithm~\ref{alg:main-alg} with different amplification factors $\eta$, while fixing $\gamma=10^{-5}$ for FashionMNIST and $\gamma=5\times 10^{-6}$ for CIFAR-10.}
\label{fig:periodicAvailableDiffAmplification}
\end{figure}

\end{document}